%% file: main.tex
\documentclass[11pt]{article} %

\def\shownotes{1}

\usepackage{macro}
\renewcommand{\cite}[1]{\citep{#1}}

\usepackage[backend=biber,
            natbib=true,
            maxbibnames=99,
            minalphanames=3,
            uniquename=false,
            uniquelist=false,            
            maxcitenames=2,
            sorting=ynt,
            giveninits=true,                        
            labeldateparts=true,
            sortcites=true,
            backref=true]{biblatex}
\usepackage[left=1.in, top=1in, bottom=1in, right=1.in]{geometry}
\DefineBibliographyStrings{english}{%
  backrefpage = {\hspace{-0.150cm}},%
  backrefpages = {\hspace{-0.100cm}},%
}
\usepackage[us,12hr]{datetime} %
\usepackage{graphicx, subcaption}
\usepackage{multirow, makecell}
\usepackage{algorithm}
\usepackage{algpseudocode}
\usepackage{bm}

\usepackage[T1]{fontenc}
\usepackage{microtype}
\usepackage{sidecap}
\usepackage{amsmath,amsfonts,amssymb}
\usepackage{graphicx}
\usepackage{booktabs} %
\usepackage{enumitem}
\usepackage{algorithm}
\usepackage{algpseudocode}

\let\algoAND\AND
\let\AND\classAND
\AtBeginEnvironment{algorithmic}{\let\AND\algoAND}

\usepackage[hidelinks]{hyperref}
\usepackage{url}
\usepackage{enumitem}
\usepackage{balance}
\usepackage[capitalize,noabbrev]{cleveref}
\hypersetup{colorlinks,linkcolor={DarkRed},citecolor={black},urlcolor={DarkRed}}

\usepackage{authblk}

\begin{document}
\title{Efficiently Learning Branching Networks for\\Multitask Algorithmic Reasoning}

\author[$\dag$]{Dongyue Li}
\author[$\dag$]{Zhenshuo Zhang}
\author[$\dag$]{Minxuan Duan}
\author[$\ddag$]{\\Edgar Dobriban}
\author[$\dag$]{Hongyang R. Zhang}
\affil[$\dag$]{Northeastern University, Boston, Massachusetts}
\affil[$\ddag$]{University of Pennsylvania, Philadelphia, Pennsylvania}
\maketitle
{\renewcommand{\thefootnote}{}
\footnotetext[0]{This is the full preprint of a paper that appeared in the Proceedings of KDD 2026. Email correspondence: \texttt{ho.zhang@northeastern.edu} and \texttt{dobriban@wharton.upenn.edu}.}}
\input{abstract}
\input{intro}

\input{content}

\begin{refcontext}[sorting=nyt]
\printbibliography
\end{refcontext}

\appendix
\input{appendix}

\end{document}

%% file: abstract.tex
\begin{abstract}
Algorithmic reasoning---the ability to perform step-by-step logical inference---is a synthetic benchmark for evaluating multi-step reasoning abilities, designed for graph neural networks and also for transformer models. Prior work has evaluated reasoning for executing a single algorithmic task, whereas a more desirable objective is to perform multiple algorithmic reasoning tasks simultaneously. We start by noting that this is inherently difficult due to differences arising from the execution traces of the algorithms (such as depth- vs. breadth-first search), which cause interference when they are trained together.

In this paper, we introduce branching neural networks, a new architecture for multitask algorithmic reasoning. The main idea is to search for a recursive tree-structured partition of $n$ algorithmic tasks into a $k$-ary tree (divided into $L$ layers). Naive search requires $O(k^{nL})$ complexity; we develop an algorithm that reduces this to $O(nL)$ by solving a convex relaxation at each layer to approximate an optimal partition. Our approach clusters these tasks using gradient-based affinity and can be used on top of any base model.

We validate our approach on algorithmic reasoning benchmarks and their extensions with text descriptions. We show that gradient-based affinity scores help estimate true performance with less than {5\%} error, measured across eight different architectures with up to 34 billion parameters. On the CLRS benchmark, our approach outperforms existing graph neural networks by {3.7}\% and baselines by {1.2}\%, while reducing runtime by {48}\% and memory usage by {26}\%. The learned branching structure shows a hierarchical clustering of related algorithms. On three text-based graph reasoning benchmarks, our approach improves over baseline methods by {3.2}\%. Finally, on a graph dataset with 21 million edges and 500 community labelings, we show a {28}\% accuracy gain over existing multitask and branching architectures, along with a {4.5$\times$} reduction in runtime.
\end{abstract}

%% file: intro.tex
\section{Introduction}

Reasoning is tied to a learning system's ability to make inductive inferences from its internal states, and it remains a central challenge in artificial intelligence \cite{feigenbaum1963computers}.
Many formalisms---most notably probabilistic reasoning via Bayesian networks---have been developed to build models that can reason efficiently under uncertainty \cite{russell1995modern}.
More recently, a complementary line of work has sought to characterize the reasoning ability of neural networks through the lens of algorithmic execution, viewing reasoning as the step-by-step computation performed by a classical algorithm.
Surprisingly, for a wide range of basic graph algorithms such as shortest paths and reachability, neural networks can be trained to accurately predict not only the final output of the algorithm but also its intermediate states  \cite{velivckovic2022clrs}.
This raises a natural next question: can we design a model capable of solving multiple algorithmic reasoning tasks simultaneously?
Since the number of possible combinations among $n$ tasks grows combinatorially, the central challenge is to develop architectures and training methods that enable efficient and scalable {multitask algorithmic reasoning}.

A rigorous study of algorithmic reasoning is valuable not only from a foundational perspective \cite{velivckovic2021neural}, but also for its potential to improve the reasoning abilities of large language models (LLMs), especially on tasks involving graphs and structured computation \cite{wang2023can}.
In our experiments, we observe that prompting open-source LLMs (e.g., Llama-3-8B and Qwen-3-1.7B) with textual descriptions of algorithmic tasks, such as shortest paths, yields performance roughly 65\% worse than directly training a graph neural network on the same task.

\begin{figure*}[t!]
    \centering
    \begin{minipage}[b]{0.995\textwidth}
        \centering
        \includegraphics[width=0.995\textwidth]{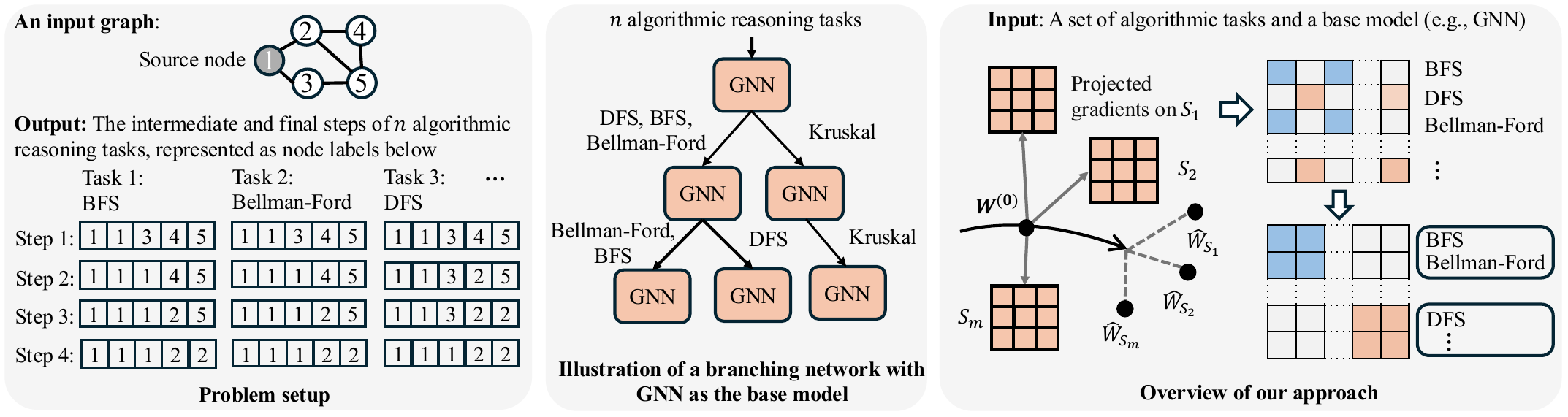}
    \end{minipage}
    \caption{An illustration of our problem setup and the proposed solution. 
    \textbf{Left:} We study learning graph algorithms such as BFS, Bellman-Ford, and DFS, formulating the prediction of intermediate algorithmic states as a node-labeling classification problem.
    \textbf{Center:} Given $n$ algorithmic reasoning tasks and a base model (e.g., GNNs or low-rank adapters), we develop an efficient method to construct a branching network that automatically learns parameter-sharing structures across all tasks.
    \textbf{Right:} Our algorithm identifies a task partition at each layer and then searches for a corresponding tree structure over layers $1, 2, \dots, L$.
    Overall, the procedure runs in time $O(nL)$, which is dramatically faster than the naive worst-case of $O(2^{nL})$.}\label{fig_pipeline}
\end{figure*}

Several existing works have sought to design neural networks for tackling a single algorithmic reasoning task, while the problem of multitask algorithmic reasoning has not been studied in depth.
\citet{velivckovic2022clrs} introduce CLRS-30, a dataset of 30 algorithmic tasks, and find that message-passing neural networks can learn to accurately predict both the intermediate and final steps of a graph algorithm (where a graph is sampled from a fixed distribution).
In particular, each intermediate step is treated as a node labeling sub-task, and the loss objective sums over all the intermediate node labeling sub-tasks (cf. equation \eqref{eq_loss}, Section \ref{sec_prelim} for the definition).
Consider breadth-first search (BFS) as an example (See also Figure \ref{fig_pipeline} for an illustration).
Starting from a graph and a source node for the search, at each intermediate step a network is trained to predict the index of the predecessor node on the current traversal path from the source.
\citet{ibarz2022generalist} conduct multitask experiments by training a single network across all CLRS-30 tasks, noting both positive and negative transfer compared to single-task training.
\citet{muller2024towards} design a tri-attention mechanism in graph transformers, and use this new architecture to improve algorithmic reasoning over existing graph neural networks on a single task.
Both works leave open the question of designing specialized optimizers for multitask algorithmic reasoning.

A key insight of this paper is that training a single neural network for multitask algorithmic tasks is inherently difficult due to complex interference between the intermediate steps (or node labels) of different algorithms.
For example, in Figure \ref{fig_pipeline}, on a toy graph example, we notice that BFS and Bellman-Ford share the same intermediate node labels, while depth-first search (DFS) differs in Steps 2 and 3.
However, if we use a single processor to predict all three tasks, this interference is unavoidable (See our study in Figure \ref{fig_graph_sc}, Section \ref{sec_sc}). In contrast, training a separate network for each task requires storing $n$ models. This increases the memory usage at inference by a factor of $n$, since $n$ networks must be evaluated.  

To address this challenge, we explore the design of branching networks for multitask algorithmic reasoning by dividing algorithms into branches that account for their similarities.
Searching over an $L$-layer, $k$-ary branching network for $n$ tasks takes $O(k^{nL})$ time in the worst case.
We introduce an efficient algorithm whose runtime is only $O(nL)$. Further, this algorithm can be applied on top of any base model, including GNNs or LLMs with low-rank adapters.
A key technical ingredient of this algorithm is to recursively estimate the affinity scores among several algorithms, conditioned on the partitioning from the previous layers.
In more detail, our algorithm involves two steps:
\begin{itemize}%
    \item First, we design an algorithm that, given a set of algorithmic reasoning tasks, partitions them into (at most) $k$ groups via convex optimization. 
    We quantify similarity scores using advances in the data attribution and influence functions literature \cite{ilyas2022datamodels,park2023trak}.
    Crucially, we develop a new notion of layer-wise affinity scores, conditioned on the partitioning from previous layers.
    The main intuition of this procedure is that we can approximate the loss of a well-trained neural network via first-order approximations, a key geometric property of over-parameterized networks \cite{jacot2018neural}.
    The runtime involves computing the gradients and functional values for the training samples of all tasks once at a pretrained initialization, leading to $O(n)$ runtime; additional computations, such as running the clustering algorithms, incur negligible cost.
    
    \item Second, we search for a branching network recursively from the first layer until layer $L$. Taken together, the overall running time of the algorithm is $O(nL)$.
    The benefit of this branching network is that the memory overhead over multitask training \cite{ibarz2022generalist} is only $\frac k L$ (typically $\le 2$), making it amenable to larger network architectures such as edge transformers \cite{muller2024towards}.
\end{itemize}
See Figure \ref{fig_pipeline} for a high-level overview of our approach.

We evaluate our approach on various algorithmic reasoning tasks across graph and text datasets. 
First, we find that the gradient-based approximation can estimate the true performance with an error of less than \textbf{5}\% for four GNNs and four LLMs with up to 34 billion parameters.
Second, we apply \acronym{} on the CLRS benchmark \cite{velivckovic2022clrs} using GNNs as the base model and observe that \acronym{} outperforms the state-of-the-art single multitask network \cite{muller2024towards} by {3.7}\% on average across twelve graph algorithmic reasoning tasks. 
\acronym{} improves over the most recent branching networks \cite{guo2020learning} and multitask networks \cite{li2024scalable} by {1.2\%} on average, while reducing GPU hours by \textbf{48\%} and GPU memory by \textbf{26\%}.
In fine-tuning language models on three text-based graph reasoning benchmarks (including CLRS-Text \cite{markeeva2024clrs}, GraphQA \cite{fatemitalk}, and GraphWiz \cite{chen2024graphwiz}), \acronym{} outperforms the strongest multitask learning baseline \cite{li2024scalable} by {3.2}\% on average. 
Additionally, we show that \acronym{} extends beyond algorithmic reasoning, applying it to overlapping community detection (on a graph with over 21M edges and 500 node-labeling tasks) \cite{whang2016overlapping}.
\acronym{} improves accuracy by {28\%} over branching and multitask networks \cite{guo2020learning,li2024scalable}, while achieving the best trade-off in runtime and memory.

\paragraph{Summary of contributions.}
This paper makes three contributions to the problem of algorithmic reasoning on graphs:
\begin{itemize}
    \item First, we develop a fast approximation procedure to efficiently identify optimal parameter-sharing paradigms using a branching network, and we empirically validate that it accurately approximates true performance across various models.
    \item Second, we design an automatic branching network that improves average performance across multiple algorithmic reasoning tasks and can be built on any base model. 
    \item Third, we evaluate our algorithm on graph- and text-based algorithmic reasoning tasks and show that it achieves competitive performance with the best trade-off between runtime and model memory.
\end{itemize}
Our work reinforces the importance of several open problems relating to algorithmic reasoning, in particular the learnability of an algorithm via neural networks \cite{velivckovic2021neural,luca2024simulation}, and generalization in algorithmic reasoning \cite{georgiev2023beyond}.
It would be interesting to use the new techniques developed in this paper to better understand the sample complexity (and in-context learnability) of algorithmic reasoning tasks.
The code for reproducing our findings can be found at \url{https://github.com/VirtuosoResearch/Multitask-algorithmic-reasoning-code}.

%% file: content.tex
\section{Preliminaries}\label{sec_prelim}

We follow the definition of algorithmic reasoning from prior work \cite{velivckovic2021neural,velivckovic2022clrs}.
Let ${A^{(1)}, A^{(2)}, \ldots, A^{(n)}}$ denote $n$ algorithms.
Our goal is to learn a neural network $f_W$ (such as a GNN) parameterized by $W$ to minimize the risk averaged over the $n$ algorithmic execution sequences.
Given an input $X$, the encoding of the $j$-th intermediate step of $A^{(i)}$ on input $X$ is denoted as $A_j^{(i)}(X)$, for $j  = 1, 2, \dots, S^{(i)}(X)$,
where $S^{(i)}(X)$ denotes the total number of execution steps by running $A^{(i)}$.

Given an input $X$, an algorithmic reasoning task involves predicting $A_j^{(i)}(X)$ as accurately as possible, for all $j = 1, \dots, S^{(i)}(X)$.
Then, let $\ell$ denote a loss function (e.g., cross-entropy) over predictions $f^{(i)}_W(X; j)$ of $A^{(i)}$ at step $j$ and ground truth labels $A_j^{(i)}(X)$.
The test loss $L(f_W)$ of $f_W$ is defined as:
\begin{align}
 \exarg{X \sim \cD}{\frac{1}{n}\sum_{i=1}^n \frac{1}{S^{(i)}(X)} \sum_{j=1}^{S^{(i)}(X)  } \ell \left( f_W^{(i)}(X; j), A_j^{(i)}(X) \right)},\label{eq_loss}
\end{align}
where $\cD$ denotes a fixed distribution from which $X$ is drawn.
For example, in the case of graphs, the input $X$ corresponds to a graph, and the labels are represented as node labels. $\cD$ can be Erd\"os-R\'enyi random graphs or preferential attachment graphs.
For each algorithm, the node labels are generated by running the algorithm on an input graph. Then, the node traversal trajectories at each step are recorded as labels. 

To help understand the above definition, consider the example illustrated in Figure \ref{fig_task_examples} on a toy graph with five nodes, for three algorithms (breadth-first search, depth-first search, and Bellman-Ford).
To understand the encoding, each intermediate step is labeled by the predecessor of each node along the traversal path, yielding a node-label vector for each step. The node labels are initially the indices of the nodes. 
\begin{itemize}%
    \item In BFS, starting from node $1$, the first step visits node $2$.
    Hence, we update node $2$'s label to $1$, its predecessor. 
    Step $2$ visits node $3$; hence, its node label is updated as $1$. 
    Step $3$ visits node $4$ (with predecessor $2$) and Step $4$ visits node $5$ (with predecessor $2$).
    \item In DFS, assuming that the smallest-numbered unvisited neighbor is chosen first, the first step visits node $2$, whose predecessor becomes $1$.
    At step $2$, from node $2$ the search moves to node $4$, whose predecessor is $2$.
    At step $3$, node $4$ moves to node $5$. Step $4$ visits node $3$ from $5$.
    \item Notice that in step $2$, BFS and Bellman-Ford follow the same traversal path, but DFS follows a different path.
\end{itemize}

\begin{figure}[!t]
    \centering
    \begin{minipage}[b]{0.80\textwidth}
        \centering
        \includegraphics[width=0.99\textwidth]{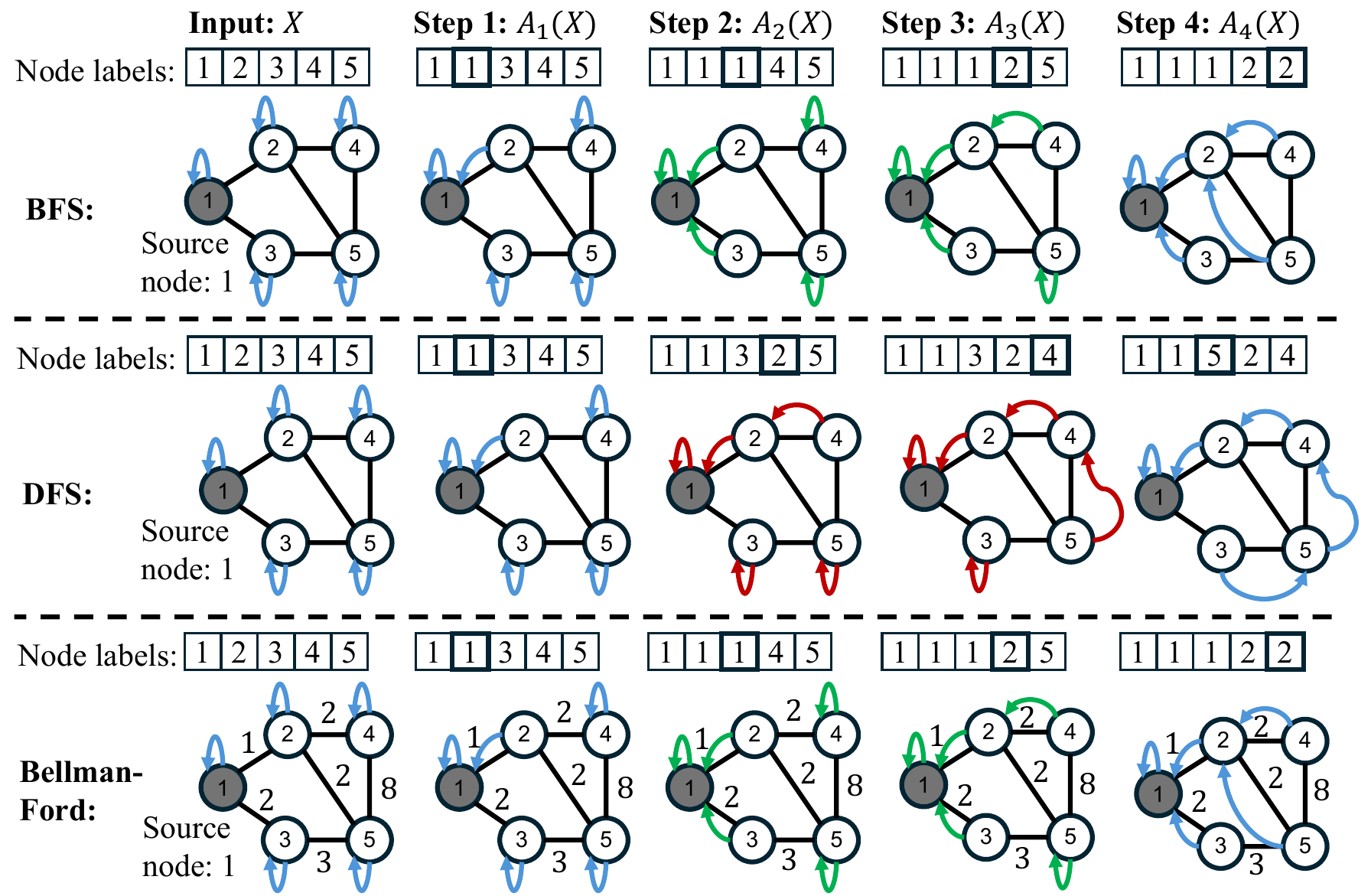}
    \end{minipage}
    \caption{We give three examples of algorithmic reasoning tasks: Breadth-first search (BFS), depth-first search (DFS), and Bellman-Ford.
    The node labels of each intermediate step encode the predecessor of each node along the traversal path. We illustrate the predecessors with arrows, and we use the integers on edges to indicate the edge weights. One can see that these algorithms share some intermediate steps, but not all, and the goal of this paper is to automatically identify such similarities.}\label{fig_task_examples}
\end{figure}

\noindent\textbf{Problem statement.} Given $n$ algorithmic reasoning tasks, can we design a neural network to predict the logical steps of all algorithms simultaneously?
Notice that in the example of Figure \ref{fig_task_examples}, BFS and Bellman-Ford share the same node labels at every step, whereas DFS shares only the first step with BFS. 
More generally, we can expect algorithms that follow a similar logic to share similar node labels in the intermediate steps.
Thus, the goal of \emph{multitask algorithmic reasoning} is to design a neural network $f_W$ that can minimize the test loss of $L(f_W)$ above.

As an extension of the above problem, suppose we are given a textual description of the same algorithmic reasoning task; the LLM's output would be to generate the correct execution steps for each algorithm.

\section{Our Approach}

We now describe a new algorithm that automatically learns one neural network for multitask algorithmic reasoning. 
Our approach involves learning a branching network that can be applied on top of any base model, enabling more flexible parameter sharing based on estimated task similarity scores. 
For example:
\begin{itemize}%
    \item For GNNs, one can instantiate multiple networks per layer and assign a specific GNN to each task at each layer. See Figure \ref{fig_tree_examples} for several examples of branching GNNs.
    \item For LLMs, one can apply parameter-efficient fine-tuning such as low-rank adapters (LoRA) \cite{hu2021lora}, and design a branching structure of LoRAs on top of a pretrained LLM.
\end{itemize}

Specifically, consider building a network with $L$ layers. 
Training a separate network for each task requires storing $n$ models, resulting in a memory cost of $O(n)$, which can be expensive for storing larger network architectures such as edge transformers \cite{muller2024towards}. 
Memory reduction can be achieved by constructing a branching network that enables parameter sharing across selected layers across tasks.
If each layer has at most $k$ possible branches, then at each layer, there are $O(k^n)$ options for assigning $n$ tasks into $k$ branches.
Iterating over layer $1$ to $L$ leads to $O(k^{nL})$.
By contrast, our approach reduces the runtime to $O(nL)$ by using a top-down search to recursively partition tasks from the first to the last layer.
Each search step, which runs in $O(n)$ time, estimates task affinity scores based on the task partition inherited from the previous layer.
We use $W$ to denote all the parameters of the branching network, and let $W_i$ denote the parameters at the $i$-th layer, for $i = 1, \dots, L$.

\begin{figure}[!t]
    \centering
    \begin{minipage}[b]{0.305\textwidth}
        \centering
        \includegraphics[width=0.65\textwidth]{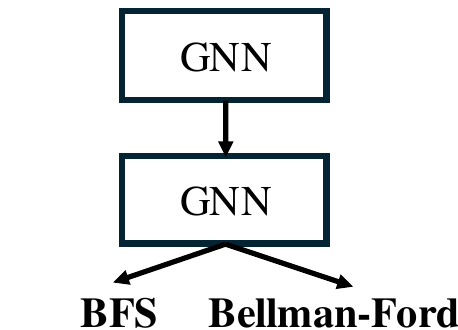}
    \end{minipage}\hfill
    \begin{minipage}[b]{0.305\textwidth}
        \centering
        \includegraphics[width=0.65\textwidth]{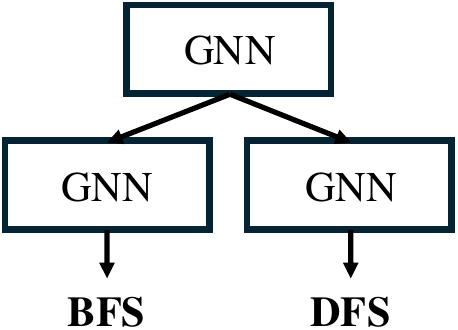}
    \end{minipage}\hfill
    \begin{minipage}[b]{0.305\textwidth}
        \centering
        \includegraphics[width=0.65\textwidth]{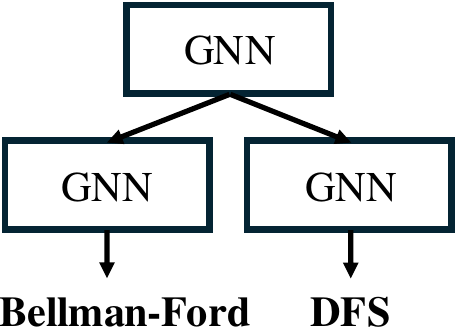}
    \end{minipage}
    \caption{We present three examples of branching GNNs, each designed to learn a pair of algorithms.
    As shown in Figure \ref{fig_task_examples}, all three algorithms share identical node labels in the first step, so the same initial GNN layer applies to all.
    BFS and Bellman–Ford continue to share node labels in steps 2 and 3, thus reusing the second layer, while DFS branches out.
    These structures are learned automatically from data: across the training graphs, BFS and Bellman–Ford share 75\% of their node labels, whereas DFS overlaps with each by 30\%. 
    }\label{fig_tree_examples}
\end{figure}

\subsection{Finding a Partitioning}

We start by describing a procedure that, given $S\subseteq \set{1, \dots, n}$ at layer $l$, determines a disjoint partition of $S$, corresponding to the branching structure at layer $l+1$. This procedure first estimates task affinity scores using partitions inherited from previous layers $1, 2, \dots, l$, and then maximizes these estimates via a convex relaxation program.

Let $S$ be a set of tasks that contain $i, j$, sampled from $\set{1,2,\dots, n}$.
Let $\hat L^{(i)}(f_{W}(S))$ denote the validation loss of task $i$ of $f_{W}(S)$ (trained on $S$).
Let $S_1, S_2, \dots, S_{m}$ denote $m$ randomly sampled subsets from $\set{1, 2, \dots, n}$.
We define the affinity score between $i, j$, conditioned on layer $l$, as:
\begin{align}\label{eq_layer_wise_task_affinity}
    T^{(i)}_{j, l} \define \frac 1 {m_{i, j}} \sum_{S_k:~ i\in S_k, j \in S_k}{\hat L^{(i)}\big(f_W(S_k)\big)},
\end{align}
where $m_{i, j}$ is the number of subsets $S_k$ where $i, j$ are both in $S_k$.
This score is analogous to the feature importance score used in random forests \cite{liidentification,li2023boosting,li2024scalable_kdd}.

Computing $T$ requires training $f_W$ $m$ times.
Instead, we design an algorithm that estimates $T$ \emph{without repeated training}.
The key idea is to use a first-order approximation of the network output around an initialization $W^{(0)}$, with respect to the weights from layer $l$ up to $L$:
\begin{align}
    f_W^{(i)}(X; j) =~ & f_{W^{(0)}}^{(i)}(X; j) + \Big[ \nabla_{{l : L}}  f_{W^{(0)}}^{(i)}(X; j) \Big]^{\top} \Big(W_{l : L} - W^{(0)}_{l : L} \Big) + \epsilon^{(i)}_{X, j},  \label{eq_approx}
\end{align}
for any step $j = 1, \dots, S^{(i)}(X)$, where we recall that $S^{(i)}(X)$ is the number of execution steps of $A^{(i)}(X)$, and $W_{l:L}$ denotes the weights from layer $l$ up to $L$.
To obtain this initialization, we first train $f_W$ from layer $l$ up to $L$ jointly on all tasks of $S$. As the partition is conditioned on the structures before layer $l$, we keep the weights in the first $l-1$ layers frozen.
This leads to an initialization $f_{W^{(0)}}$.
Our experiments (See Table \ref{table_compare_approximation_error}, Section \ref{sec_approx}) show that this approximation yields less than 5\% relative errors on algorithmic reasoning tasks across four GNNs and four LLMs.
This estimation leverages the fact that a well-trained neural network uses its gradients as features \cite{jacot2018neural}.
Hence, the linear approximation is accurate around a local vicinity of the initialization \cite{li2024scalable_kdd}.

\begin{algorithm}[t!]
    \caption{\fastapprox{} (Fast Approximate Partitioning)}\label{alg_layerwise_tag}
    \raggedright
    \textbf{Input}: Training/validation data for a subset of tasks $S$; Layer $l$ \\
    \textbf{Require:} Network parameters $W$; Number of subsets $m$ and their sizes $\alpha$; Projection dimension $d$; Number of groups $g$ \\ %
    \textbf{Output}: A disjoint partition of $S$, a trained initialization $f_{W^{(0)}}$\\
    \begin{algorithmic}[1]
        \State $f_{W^{(0)}} \leftarrow$ A meta-initialization trained from $W$ on $S$, with layers $1, 2, \dots, l-1$ fixed
        \State $S_1, S_2, \ldots, S_m \leftarrow$ $m$ random subsets of $S$ with size $\alpha$
        \State $W_{l:L} \in \real^p \leftarrow$ Parameters of layers $l$ to $L$
        \State $P \leftarrow$ A $p$ by $d$ isotropic Gaussian random projection matrix
        \For{each training sample $X$, $j = 1, \dots, S^{(i)}(X)$}
            \State $\tilde{g}_j^{(i)} \leftarrow P^{\top} \nabla_{W_{l : L}} f_{W^{(0)}}^{(i)}(X; j)$, $\forall i\in S$ %
        \EndFor
        \For{$k = 1, \dots, m$}
            \State $\hat W_{d} \leftarrow \arg\min \sum_{i \in S_k} \sum_{X} \tilde\ell(f_W^{(i)}(X))$ %
            \State $\hat W_{S_k} \leftarrow W^{(0)} + P \hat W_d$
            \State $\hat L^{(i)}\Big(f_{\hat W_{S_k}}(S_k)\Big) \leftarrow$ Validation loss of $f_{\hat W_{S_k}}$ on $A^{(i)}$, $\forall i \in S_k$
        \EndFor
        \State $T \leftarrow$ $\abs{S}$ by $\abs{S}$ affinity score matrix via equation \eqref{eq_layer_wise_task_affinity}
        \State $S_{1}, S_{2}, \ldots, S_{g} \leftarrow$ A disjoint partition of $S$ via  clustering on $T$ 
        \State Return $(S_{1}, l+1), (S_{2}, l+1), \ldots, (S_{g}, l+1)$; $f_{W^{(0)}}$
    \end{algorithmic}
\end{algorithm}

Building on the above approximation, we propose to estimate $\hat L^{(i)}(f_W(S_k))$, for any $k = 1, \dots, m$, given $W^{(0)}$.
We estimate the model weights trained on a subset of tasks (denoted as $\hat{W}_{S_k}$) using the gradients evaluated at the initialization.  
For simplicity, we illustrate the case of binary classification, and the same logic applies to multi-class and regression problems. %
By applying the first-order approximation in equation \eqref{eq_approx} to the training loss, we solve a logistic regression problem using the gradients as features, which minimizes the approximated log-loss on the training examples from a subset of tasks: %
\begin{align*}
    \tilde \ell(f_W^{(i)}(X)) = \sum_{j=1}^{S^{(i)}} \log \Big( 1 + \exp \big( - A_j^{(i)}(X) \big[g_j^{(i)}\big]^\top \Big( W_{l:L} - W^{(0)}_{l:L} \Big) 
     -A_j^{(i)}(X) f_{W^{(0)}}^{(i)}(X; j) \big) \Big),
\end{align*}
where $g_j^{(i)} = \nabla_{{l:L}} f_{W^{(0)}}^{(i)}(X; j)$.
In practice, we use random projections of the gradients, following the Johnson-Lindenstrauss lemma.
Solving the above logistic regression problem, we evaluate the model with the estimated model weights $\hat{W}_{S_k}$ to obtain $\hat L^{(i)}(f_W(S_k))$.  

Repeating the estimation to $m$ randomly sampled subsets leads to an approximation of $T^{(i)}_{j, l}$ after averaging over subsets including tasks $i$ and $j$.
The output of this step is an $|S| \times |S|$ task affinity matrix.
Crucially, this step runs in $O(|S|)$ time, which involves training $W^{(0)}$ and evaluating the gradients on the samples in $S$.
Additional costs include running $m$ logistic regressions, which can be performed on CPUs and completed in a few seconds.

Finally, we apply a convex relaxation program to maximize the average affinity score density within clusters.
We adjust the cluster size by regularizing the trace of the assignment variable.
The details are provided in Appendix \ref{app_clustering}.
This clustering step takes less than a few seconds.
In practice, we can apply the clustering step several times to find the optimal cluster size.
In summary, the partitioning procedure at one layer is described in Algorithm \ref{alg_layerwise_tag}.

\subsection{Learning Branching Networks}

Next, we search for a branching network via a top-down procedure.
The algorithm begins with a single network with one module per layer.
Starting at $l=1$, suppose that tasks are grouped into $k_1$ clusters. This creates $k_1$ modules at layer $1$. If $k_1 = 2$, tasks are split into two groups, denoted as $S_{1}$ and $S_{2}$. The procedure continues recursively: Each group is further split at the next layer.
If both are split into two groups, the second layer then contains four modules, denoted as $S_{3}$, $S_{4}$, $S_{5}$, and $S_{6}$.
This continues until layer $L$.
The full procedure is in Algorithm \ref{alg_building_branching_network}.

In terms of running time, at each layer the algorithm takes $O(n)$ time to find a partition, since the union of the sets is at most $n$. In total, Algorithm \ref{alg_building_branching_network} takes $O(nL)$ time.
Regarding memory usage, 
suppose the last layer contains $k$ clusters, and at each layer, the number of clusters grows by a constant factor $a$.
Then the total number of nodes in the tree is roughly $\sum_{i=0}^{L-1} a^i = \frac{k^{L/(L-1)} - 1} {k^{1/(L-1)} - 1}$, where $a^{L-1} \approx k$.

Let the running time for training a single $L$-layer network on a single task be $T$, and the memory usage be $B$. In summary, our algorithm runs in time $nLT$ and uses approximately $\frac{kB} L$ memory to store the network.
For comparison, training a separate network for each task (i.e., single-task learning) requires $nT$ time and $nB$ memory to store $n$ models.
A mixture-of-experts architecture with $k$ networks requires $knT$ time and $kB$ memory.
Task affinity grouping \cite{fifty2021efficiently} with fully-computed affinity scores takes $n^2T$ time and $kB$ memory.
LearningToBranch \cite{guo2020learning} takes $k^LnT$ time and uses $nB$ memory since it trains with $n$ modules per layer.
This comparison is summarized in Table \ref{tab_related_works_comparison}.

\begin{algorithm}[t!]
    \caption{\acronym~(Automated BRAnching NEtworks)}\label{alg_building_branching_network}
    \raggedright
    \textbf{Input}: Training and validation datasets for $n$ algorithmic tasks\\
    \textbf{Output}: A branching neural network $f_W$\\
    \begin{algorithmic}[1]
    \State $f_W\leftarrow$ An $L$-layer network with initialized parameters $W$, with one model at each layer
    \State $Q \leftarrow \set{(\set{1, 2, \dots, n}, {1}) }$ %
    \While{$Q\neq\varnothing$}
        \State $S~\leftarrow$ Dequeue one subset from $Q$ with layer index $l$ 
        \If{$l < L$}
            \State $(S_{1}, l+1), \dots, (S_{k}, l+1)$; $f_{W^{(0)}}~\leftarrow$ \fastapprox{}\big($S, l; W$\big)
            \State ${W}_l \leftarrow {W}_l \cup \Big\{{W^{(0)}_l}\Big\}$
            \Comment{Update the branching network at layer $l$ with weights trained on $S$}
            \State $Q\leftarrow Q \cup \big \{(S_{1}, l+1), \dots, (S_{k}, l+1)\big \}$
        \EndIf
    \EndWhile
    \State \textbf{Return} $f_W$
    \end{algorithmic}
\end{algorithm}

\begin{remark}
    Our approach does not assume that tasks exhibit hierarchical dependencies. The branching design applies when tasks exhibit various levels of similarity, where more similar tasks share more
    layers, and dissimilar tasks share fewer. The procedure can, in principle, handle any type of task dependencies.

    Additionally, this approach can be extended to incorporate new tasks into the branching network. One way is to determine the group partition of the new tasks, layer by layer, by incrementally fine-tuning each layer on the new tasks and estimating its task affinities. Our approximation procedure can also be extended to handle this extension.
\end{remark}

\begin{remark}
For LLMs, our approach constructs a single branching structure of LoRA adapters on top of each model, which can be computed efficiently using a fast approximation algorithm. Note that the approach can be applied with other parameter-efficient fine-tuning methods besides LoRA to construct the adapters. 
\end{remark}

\begin{table}[t!]
\centering
\caption{Summary of runtime and memory cost of \acronym{} and existing multitask and branching networks. Here $n$ is the number of tasks, $L$ is the number of layers, and $k$ is the number of branches at the last layer. $T$ denotes the runtime of training a single network of $L$ layers for one model. $B$ denotes the memory usage of one base model.}\label{tab_related_works_comparison}
{\small
\begin{tabular}{@{}l c c c c @{}}
\toprule
\textbf{Methods} &  \textbf{Runtime} & \textbf{Memory} \\ \midrule 
Single task learning (STN) & $nT$ & $nB$ \\ 
Multitask learning (MTN) & $nT$ & $B$ \\
Multi-gate MoE (MMoE) \cite{ma2018modeling,hazimeh2021dselect} & $knT$ & $kB$ \\
Task affinity grouping (TAG) \cite{fifty2021efficiently,li2023boosting} & $n^2T$ & $kB$ \\
LearningToBranch \cite{guo2020learning} & $k^L n T$ & $nB$ \\
\acronym{} (Algorithm \ref{alg_building_branching_network}) & $\bm{nLT}$ & $\bm{\frac{k}{L}B}$ \\
\bottomrule
\end{tabular}
}
\end{table}

\subsection{Theoretical Analysis}

In this section, we show that the error introduced by the approximation in \fastapprox{} is bounded. 
We will show that the training loss of the estimated parameters $\hat{W}_{S}$ by solving the logistic regression is bounded. 

\begin{proposition}\label{prop_error_bound}
Let $\cD$ be a search space of model weights $W$ whose radius is at most $D$.
Suppose the gradient of $f_{W^{(0)}}$ at the initialization $W^{(0)}$ in the training set is at most $G$ in Euclidean norm.
Let $T$ be the training set of inputs. For each algorithmic reasoning task $i \in \set{1, 2, \dots, n}$, the training data is collected by executing the algorithm on each input. Suppose that for every task $i$, the average error of the linear approximation is bounded as follows:
\begin{align*}
    \frac{1}{\abs{T}} \sum_{X \in T} \frac{1}{S^{(i)}(X)} \sum_{j=1}^{S^{(i)}(X)} \epsilon_{X, j}^{(i)} \le \delta,
\end{align*}
where
\begin{align*}
    \epsilon_{X, j}^{(i)} = \Bigg\lvert f_W^{(i)}(X; j) - f_{W^{(0)}}^{(i)}(X; j) -  \Big[ \nabla  f_{W^{(0)}}^{(i)}(X; j) \Big]^{\top} \Big(W - W^{(0)} \Big) \Bigg \rvert. 
\end{align*}
Let $\hat{L}_S(f_W)$ be the training loss of a subset of tasks $S \subseteq \set{1,2,\dots,n}$ defined as:
\begin{align*}
    \hat{L}_S(W) = \frac{1}{\abs{T}\abs{S}}\sum_{X \in T} \sum_{i\in S} \frac{1}{S^{(i)}(X)} \sum_{j=1}^{S^{(i)}(X)  } \ell \left( f_W^{(i)}(X; j), A_j^{(i)}(X) \right).
\end{align*}
Provided that the projection dimension $d = O\Big(\frac{\log p}{\epsilon^2}\Big)$, the training loss of $\hat W_S$ is bounded away from the minimum training loss for any $S$ as
\begin{align}\label{eq_error_bound}
    \hat L_S(\hat W_S) \le \min_{W\in\cD} \hat L_S(W) + 2\delta +  {2 G D } {\epsilon}.
\end{align}
\end{proposition}
The proof is based on using the Johnson-Lindenstrauss Lemma \cite{johnson1984extensions}. The idea is to relate the loss $\hat L_S(\hat W_S)$ and $\min_{W\in\cD} \hat L_S(W)$ by the $1$-Lipschitz continuous property of the logistic loss. The error in the random projection is bounded by the JL Lemma, and the error in the linear approximation is bounded by Taylor's expansion error $\delta$. As $\epsilon$ goes to zero, equation \eqref{eq_error_bound} guarantees the gap between $\hat L_S(\hat W_S)$ and $\min \hat L_S(W)$ will be small.
We note that this error bound also applies when approximating certain layers by restricting the search space of the parameters in those layers.

\section{Experiments}\label{sec_experiments}

We now evaluate \acronym{} on algorithmic reasoning tasks in various settings. The evaluation focuses on the following questions.
First, how accurate is the first-order loss approximation for algorithmic reasoning tasks?
Second, how well does \acronym{} perform on the CLRS benchmark, and what hierarchical structures does it discover?
Third, how effectively does \acronym{} generalize to text-based graph reasoning and large-scale multitask scenarios? 

Our experiments show that the first-order approximation estimates the losses of fully trained networks with a relative error of less than 5\% across four GNNs and four language models.
On twelve algorithmic tasks from the CLRS benchmark, \acronym{} outperforms state-of-the-art networks by an average of \textbf{3.7}\% and the strongest multitask baselines, LearningToBranch \cite{guo2020learning} and GradTAG \cite{li2024scalable_kdd}, by \textbf{1.2}\%, while reducing GPU hours by \textbf{48}\% and memory usage by \textbf{26}\%.
\acronym{} reveals branching structures that group tasks with similar intermediate steps.
On three text-based graph-reasoning benchmarks, it achieves up to a \textbf{3.2}\% gain over the strongest multitask baseline.
Finally, on a large-scale community dataset comprising 500 labeling tasks and 21M edges, \acronym{} achieves a \textbf{28}\% relative improvement over the baselines, with the best trade-off between runtime and memory usage.

\subsection{Experimental Setup}

\subsubsection{Datasets and Models} 

We evaluate \acronym{} on various algorithmic reasoning tasks in both graph and text formats, focusing on performance, runtime, and memory trade-offs.
First, we use the graph-format datasets of twelve graph algorithmic reasoning tasks in the CLRS benchmark ~\cite{velivckovic2022clrs}. 
The full list of tasks is described in Table \ref{tab_clrs_tasks}. 
Following the protocol of CLRS, for each algorithm, we use a training dataset of 1,000 Erd\"os-R\'enyi graphs with 16 nodes and a validation set of 32 graphs with 16 nodes. 
We then evaluate performance on a test set of 32 graphs with 64 nodes. On CLRS, we evaluate \acronym{} with the edge transformer~\cite{muller2024towards}, which achieves state-of-the-art performance. We also evaluate \acronym{} using a standard MPNN~\cite{velivckovic2022clrs} and graph attention networks~\cite{velivckovic2018graph}.

\begin{table*}[t!]
\centering
\caption{Definitions of input and intermediate steps of $12$ graph-based algorithmic reasoning tasks from the CLRS benchmark. The input graphs are sampled from an Erd\"os-R\'enyi random graph distribution with edge sampling probability $p = 0.5$, with 16 vertices in total on each graph. The edge weights are uniformly sampled from $[0,1]$. The source node is sampled randomly from the entire vertex set.
}\label{tab_clrs_tasks}
\resizebox{\textwidth}{!}
{
\begin{tabular}{@{} p{4.0cm} p{5cm} p{3cm} p{8cm} @{}}
\toprule
Task & Input Graph & Additional Input & Encodings of Intermediate Steps \\ 
\midrule
Breadth-first search & Undirected & Source node & Node-level labels indicating its predecessor \\ 
Depth-first search & Directed & Source node & Node-level labels indicating its predecessor \\ 
Topological sort & Directed acyclic & None & Node-level label indicating its predecessor \\
Articulation points & Undirected & None & Node-level binary label for the articulation point \\ 
Bridges & Undirected & None & Edge-level binary label for the bridge \\ 
SCC Kosaraju & Directed & None & Node-level label for the SCC index \\
MST Kruskal & Undirected, weighted & None & Edge-level binary label for MST \\ 
MST Prim & Undirected, weighted & Root of MST & Edge-level binary label for MST \\ 
Dijkstra & Directed, weighted & Source node & Node-level labels indicating its predecessor \\ 
Bellman-Ford & Directed, weighted & Source node & Node-level labels indicating its predecessor \\ 
DAG shortest paths & Directed, acyclic, weighted & Source node & Node-level labels indicating its predecessor \\ 
Floyd-Warshall & Undirected, weighted & None & Edge-level labels indicating the source node \\ 
\bottomrule

\end{tabular}
}
\end{table*}

Second, we evaluate \acronym{} on text-based graph algorithmic reasoning tasks, where we construct a branching network of LoRA adapters to fine-tune LLMs.
We first use the CLRS-Text benchmark~\cite{markeeva2024clrs}, which encodes graphs as flattened adjacency lists and represents outputs as the intermediate steps from CLRS. We use the same twelve graph algorithms as in CLRS, with 1,000 training graphs of 10 nodes and 200 graphs each for validation and testing.
Second, we evaluate on GraphQA~\cite{fatemitalk}, which comprises text-based graph reasoning tasks, including edge existence, node degree, and cycle detection. We use twelve tasks from this benchmark. 
Third, we evaluate on the datasets from GraphWiz \cite{chen2024graphwiz}, which constructs textual descriptions of the reasoning processes for solving graph-related tasks by prompting from GPT-4. We use the nine tasks from this dataset. 
The description of the tasks, their inputs, and outputs is provided in Appendix~\ref{sec_experiment_details}.
For the datasets, we fine-tune Qwen-3-1.7B and Llama-3-1B using LoRA~\cite{hu2021lora} with rank 16.

Lastly, we evaluate \acronym{} on a large-scale multi-label prediction task involving 500 labels. We use the Orkut social network dataset from SNAP~\cite{yang2013overlapping}, which contains 395K nodes, 21M edges, and 500 community labels. Each community is a subgraph of a graph.  Given a partially labeled subgraph, the task is to predict the remaining node labels. For each community, we sample 10\% of the nodes for training, 20\% for validation, and use the remaining nodes for testing. As the encoder, we use a 3-layer SIGN model~\cite{frasca2020sign} with 256 hidden units, an efficient variant of GCNs.

\subsubsection{Metrics}
In the CLRS benchmark, we follow the official protocol and report task scores on the test set using the metric specified for each task, such as the $F_1$ score for node-label predictions. The $F_1$-scores are evaluated between model outputs and task labels at the final step. 
For text-based graph reasoning tasks, we evaluate test accuracy between model outputs and task labels at every intermediate step. 
In community detection, we evaluate the average macro $F_1$-score over 500 tasks.

We compute the average task score across metrics such as accuracy or $F_1$-score over all tasks and report the error rate as one minus the average task score. We measure the efficiency in GPU hours and GPU memory usage in GB. 

\subsubsection{Implementations} Recall that \acronym{} requires setting the number of subsets $m$, the subset size $\alpha$, the gradient projection dimension $d$, and the number of clusters $g$. To determine the task clusterings, we tune the hyperparameters sequentially, one at a time.
We begin by increasing the number of subsets until the affinity scores converge. For task counts $n$ from $12$ to $500$, we vary $m$ between $200$ to $5000$ and vary $\alpha$ between $3$ to $25$. We observe that affinity scores typically stabilize once $m$ reaches over $10n$. 
Next, we compute task affinities using the linear approximation technique, varying the dimension of the gradient projection $d$ from $200$ to $1000$. We find that $d$ greater than $400$ is sufficient to yield approximation errors of less than 5\%.  
After obtaining affinity scores, we tune the cluster size and select the configuration with the highest average within-cluster affinity score, which can be done efficiently by running the clustering step. We discuss key parameters in deciding the branching structure in Section \ref{sec_ablation}.
Other hyperparameters involved are described in Table \ref{tab_hyper_params}.

\begin{table*}[t!]
\centering
\caption{List of hyperparameters for each dataset tested in the experiments.}\label{tab_hyper_params}
{\small
\begin{tabular}{@{}lccccccccc@{}}
\toprule
Dataset &  Model & \makecell{$m$} &  $\alpha$ & $L$ & $k$ & $T$ & $B$ & Learning rate & Epochs \\ \midrule
CLRS & Edge Transformer & 200 & 3 & 5 & 10 & 2.5 hours & 10.5 GB & 2$e^{-4}$ & 8 \\ 
CLRS  & MPNN & 200 & 3 & 5 & 10 & 0.3 hours & 4.2 GB & 2$e^{-4}$ & 8 \\ 
CLRS-Text & Qwen-3-1.7B & 200 & 3 & 28 & 7 & 6.0 hours & 13.6 GB & 2$e^{-5}$ & 10\\
GraphQA & Llama-3-1B & 200 & 3 & 16 & 8 & 0.5 hours & 12.4 GB & 1$e^{-5}$ & 10 \\
GraphWiz & Llama-3-1B & 200 & 3 & 16 & 8 & 0.9 hours & 21.5 GB & 1$e^{-5}$ & 10 \\
Orkut & SIGN & 5000 & 25 & 3 & 22 & 0.1 hours & 5.7 GB & 1$e^{-2}$ & 100\\
\bottomrule
\end{tabular}
}
\end{table*}

\subsection{Approximation Results}\label{sec_approx}

First, we assess the accuracy of the first-order approximation used in our search algorithm. Recall that we apply the approximation to the model loss around an initialization $W^{(0)}$. We show that the error term induced by the approximation is small. Given the input $X$ and an algorithm $A^{(i)}$, we compute the average residual sum of squares (RSS) across all output steps as follows:
\begin{align*}
    \frac{1}{s^{(i)}(X)} \sum_{j=1}^{s^{(i)}(X)} \frac{\bignorm{f_W^{(i)}(X; j) - f^{(i)}_{W^{(0)}}(X; j) - \big[g_j^{(i)}\big]^\top \big( W - W^{(0)} \big) }^2}{\bignorm{f_W^{(i)}(X; j)}^2}. %
\end{align*}

We conduct experiments across various models, using twelve graph algorithmic reasoning tasks from both the CLRS benchmark \cite{velivckovic2022clrs} and its text-format version \cite{markeeva2024clrs}.
On graph-format datasets, we test four GNN architectures, including a standard MPNN \cite{velivckovic2022clrs}, GAT \cite{velivckovic2018graph}, Triplet-GMPNN \cite{ibarz2022generalist}, and Edge Transformer \cite{muller2024towards}. 
On text-format datasets, we test four language models, including Qwen-3-1.7B, Llama-3-8B, Qwen-3-14B, and CodeLlama-34B. We use LoRA \cite{hu2021lora} as the base fine-tuning procedure. 

To evaluate the RSS, we first obtain an initial model $W^{(0)}$ by training on all datasets. We then fine-tune the model on random task subsets to obtain $W$ and evaluate the RSS under various relative distances $\frac {\norm{W - W^{(0)}}} {\norm{W}}$. For language models, $W$ denotes the parameters of the LoRA adapters. Across 50 random task subsets, we find that the fine-tuned weights remain close to the initialization, typically within 10\% relative distance.

The results are shown in Table \ref{table_compare_approximation_error}.
For all four GNN architectures, the first-order approximation yields under \textbf{1}$\%$ relative error when the distance from initialization is within 10\%.
On language models, the error remains below \textbf{5}\%, with larger models generally yielding lower errors.
We further observe that freezing early layers improves approximation accuracy. When the first three layers are frozen, the RSS drops below \textbf{0.06}\%, shown in Appendix \ref{sec_additional_approximation_results}. 
This supports the premise of our algorithm: avoiding repeated training on task subsets.

\begin{table}[t!]
\centering
\caption{We measure the approximation error of losses on algorithmic reasoning datasets using a first-order Taylor expansion around the initialization trained on all tasks. Results are averaged over 50 random subsets of algorithmic reasoning tasks. We measure the errors across four GNNs, including MPNN \cite{velivckovic2022clrs}, GAT \cite{velivckovic2018graph}, Triplet-GMPNN \cite{ibarz2022generalist}, and Edge Transformer \cite{muller2024towards}. We measure four LLMs with up to 34 billion parameters for text-based reasoning tasks (by applying the approximation to LoRA parameters).}\label{table_compare_approximation_error}
{\small\begin{tabular}{@{}cccccc@{}}
\toprule
 & MPNN & GAT & Triplet-GMPNN & Edge Transformer \\ \midrule
Dist. & RSS & RSS & RSS & RSS \\ \midrule
2\%  &  2.3$_{\pm 0.7} \times 10^{-3}$  & $1.2_{\pm0.1}  \times 10^{-3}$ & $3.6_{\pm0.2}  \times 10^{-4}$  & 3.9$_{\pm 0.4} \times 10^{-3}$  \\
4\%  & 7.0$_{\pm 0.8} \times 10^{-3}$  & $1.3_{\pm0.2}  \times 10^{-3}$ & $1.6_{\pm0.2} \times 10^{-3}$ & 6.5$_{\pm 0.5} \times 10^{-3}$ \\
6\% & 8.2$_{\pm 1.4} \times 10^{-3}$  & $1.6_{\pm0.4}  \times 10^{-3}$ & $3.9_{\pm0.7}  \times 10^{-3}$ &  7.7$_{\pm 1.6} \times 10^{-3}$ \\
8\% & 8.6$_{\pm 1.2} \times 10^{-3}$  & $2.3_{\pm1.4}  \times 10^{-3}$ & $4.1_{\pm0.3}  \times 10^{-3}$ & 8.4$_{\pm 2.4} \times 10^{-3}$\\
10\% & 9.2$_{\pm 2.4} \times 10^{-3}$ &  $3.8_{\pm2.0}  \times 10^{-3}$ & $7.9_{\pm1.2} \times 10^{-3}$ & 9.6$_{\pm 1.1} \times 10^{-3}$  \\ \midrule
& Qwen-3-1.7B &  Llama-3-8B & Qwen-3-14B & CodeLlama-34B \\
\midrule
Dist. & RSS & RSS & RSS & RSS \\ \midrule
2\%  & 3.9$_{\pm 1.8} \times 10^{-3}$ &  4.8$_{\pm 1.9} \times 10^{-3}$  & 2.0$_{\pm1.1}  \times 10^{-3}$ & 2.0$_{\pm1.8} \times 10^{-3}$  \\
4\%  & 5.5$_{\pm 2.1} \times 10^{-3}$ & 4.9$_{\pm2.0} \times 10^{-3}$  & 2.1$_{\pm1.4}  \times 10^{-3}$ & 2.0$_{\pm1.7} \times 10^{-3}$ \\
6\% & 1.1$_{\pm 0.4} \times 10^{-2}$ & 5.5$_{\pm2.0} \times 10^{-3}$  & 2.5$_{\pm1.6}  \times 10^{-3}$ & 2.1$_{\pm1.7} \times 10^{-3}$ \\
8\%  & 2.3$_{\pm 0.9} \times 10^{-2}$ & 7.1$_{\pm2.0} \times 10^{-3}$  & 2.6$_{\pm1.8}  \times 10^{-3}$ & 2.4$_{\pm1.7} \times 10^{-3}$\\
10\% & 4.6$_{\pm 1.6} \times 10^{-2}$ & 9.0$_{\pm3.0} \times 10^{-3}$ &  3.0$_{\pm1.6}  \times 10^{-3}$ & 2.7$_{\pm1.9} \times 10^{-3}$  \\
\bottomrule
\end{tabular}}
\end{table}

\subsection{Comparison Results}

\subsubsection{Baselines}
We compare \acronym{} against multitask optimization baselines.
Multitask network (MTN): This trains a single shared network on all tasks. 
Multi-gate Mixture-of-Experts model (MMoE) \cite{ma2018modeling}: This uses task-specific gating networks to combine outputs from multiple expert models. 
Task Affinity Grouping (TAG) \cite{fifty2021efficiently}: This clusters tasks based on pairwise affinities and trains a separate model for each group.
LearningToBranch \cite{guo2020learning}: This is the most recent branching network that learns branching decisions using the Gumbel-Softmax technique.  
GradTAG \cite{li2024scalable_kdd}: This is the most recent multitask learning approach that estimates task affinities based on gradients and trains separate networks per group, without a branching structure. 
For the baselines, we select network size based on average validation performance across tasks.

\subsubsection{Results on GNNs} 

Shown in Figure~\ref{fig_clrs_results}, \acronym{} achieves the best trade-off between performance, runtime, and memory among all baselines.
Compared with a single Edge Transformer, which previously achieved the state-of-the-art on CLRS, our method improves average accuracy by \textbf{3.7}\%.
Relative to the strongest multitask baselines, including LearningToBranch~\cite{guo2020learning} and GradTAG~\cite{li2024scalable_kdd}, \acronym{} yields a \textbf{1.2}\% accuracy gain while reducing GPU hours by \textbf{48}\% and memory usage by \textbf{26}\%.
We observe similar performance gains when using MPNN as the backbone. Our approach achieves \textbf{1.2}\% accuracy gains over the best multitask baselines, while reducing GPU hours by \textbf{50}\% and GPU memory usage by \textbf{29}\%. 
We report the full results in Table \ref{tab_full_clrs_benchmark_results}.

\begin{figure*}[t!]
    \begin{minipage}[b]{0.24\textwidth}
        \centering
        \includegraphics[width=0.99\textwidth]{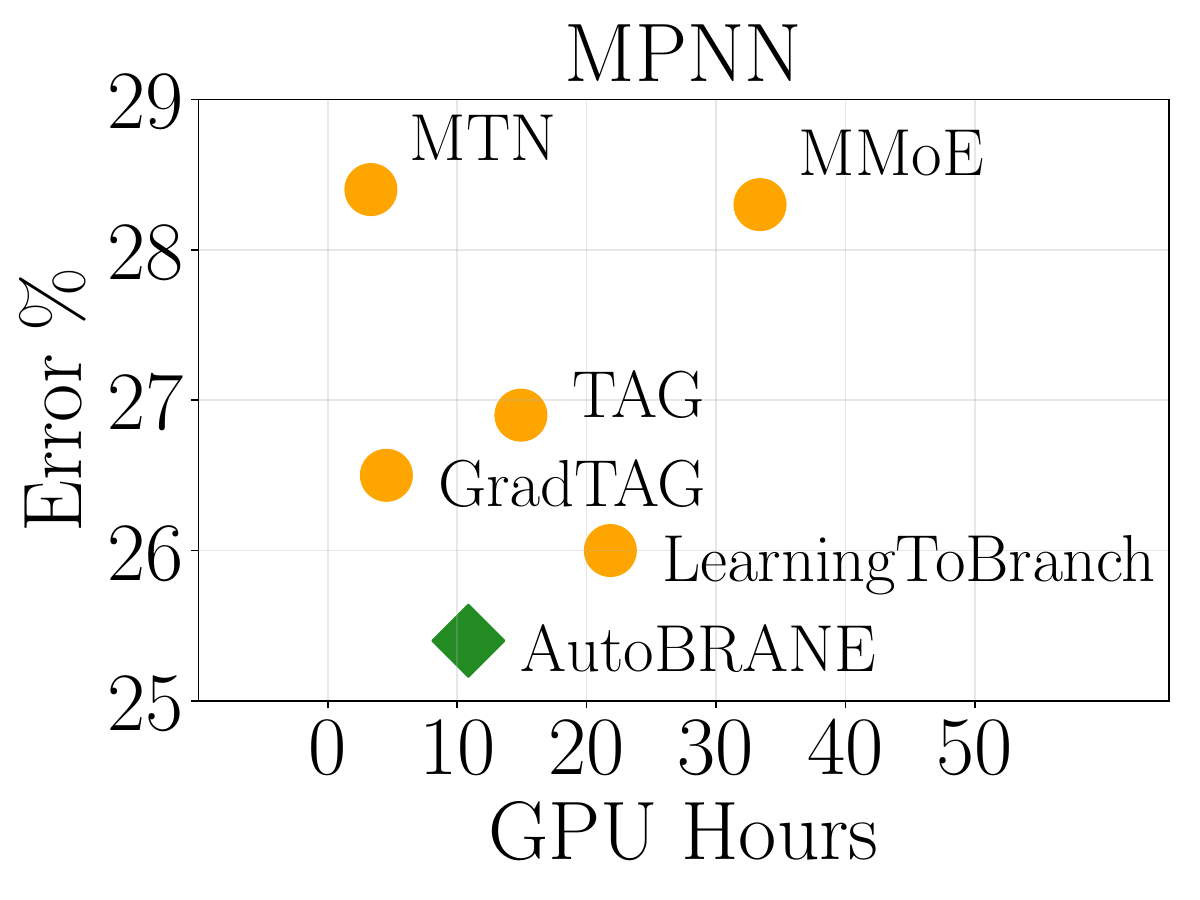}
    \end{minipage}
    \begin{minipage}[b]{0.24\textwidth}
        \centering
        \includegraphics[width=0.99\textwidth]{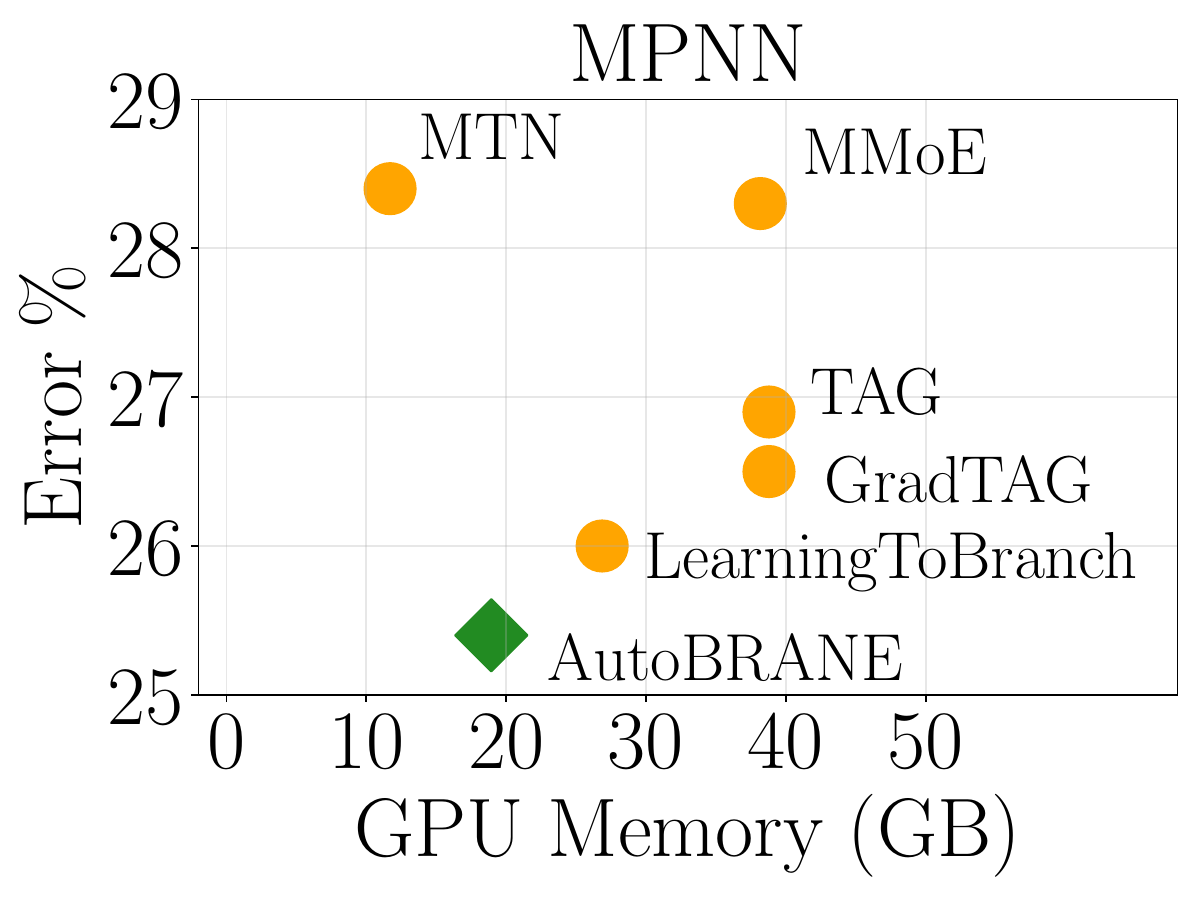}
    \end{minipage}
    \begin{minipage}[b]{0.24\textwidth}
        \centering
        \includegraphics[width=0.99\textwidth]{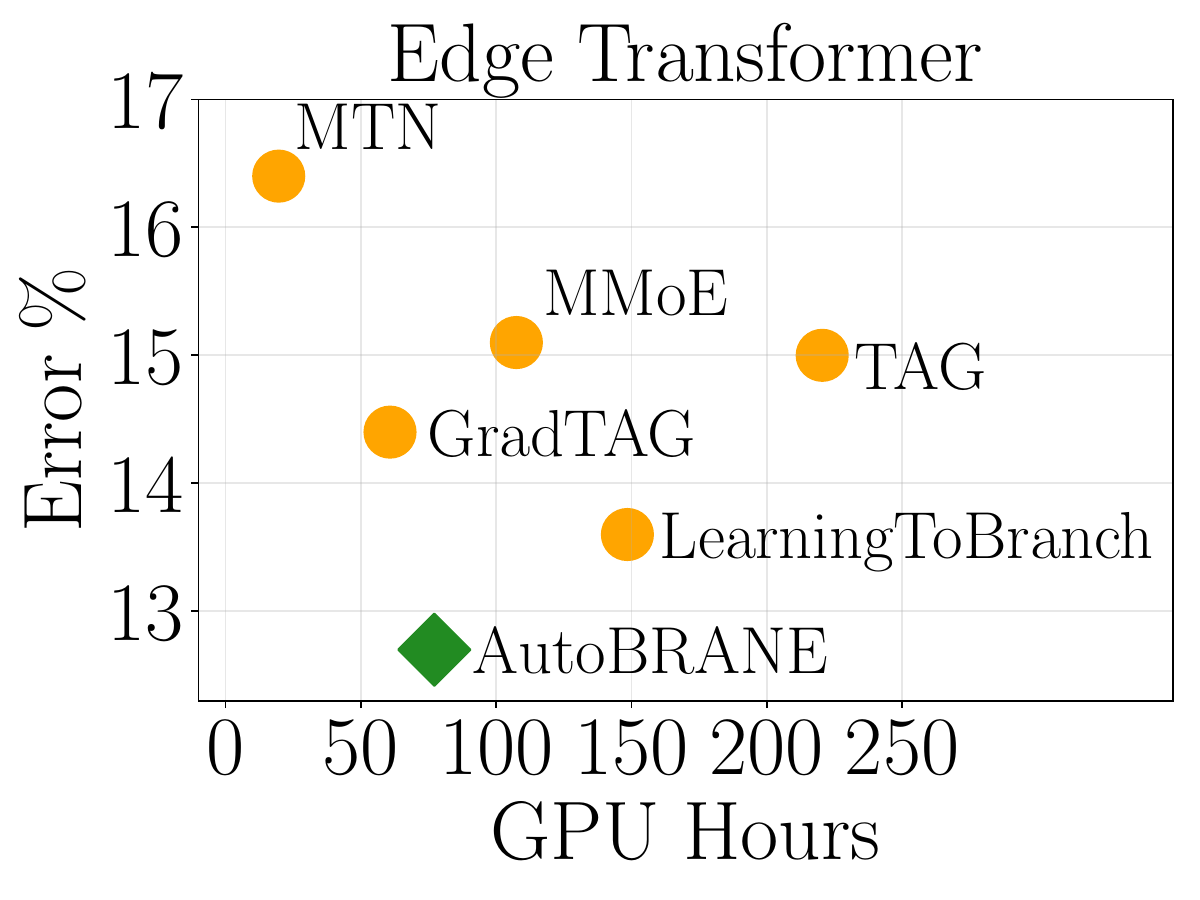}
    \end{minipage}
    \begin{minipage}[b]{0.24\textwidth}
        \centering
        \includegraphics[width=0.99\textwidth]{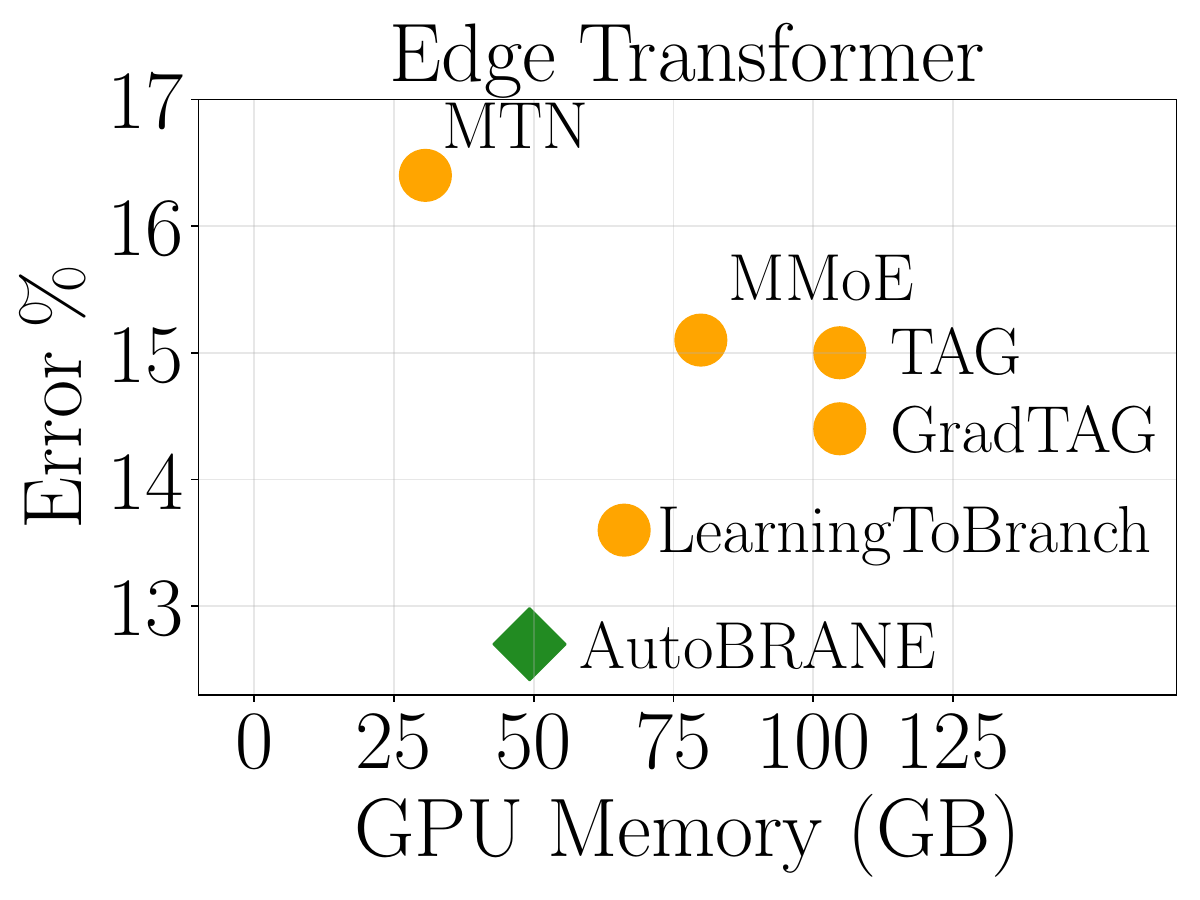}
    \end{minipage}
    \caption{Illustration of the trade-off between error rate, GPU hours, and memory usage for \acronym{} compared to existing multitask and branching network baselines. We present results obtained using MPNNs or edge transformers~\cite{muller2024towards} as the base model. 
    \acronym{} outperforms a single multitask network by 3.7\%, demonstrating the effectiveness of branching networks in leveraging positive task transfer. It also achieves the best overall trade-off, reducing the average error rate by 1.2\% compared to the strongest baseline, while using 48\% fewer GPU hours and 26\% less memory.}\label{fig_clrs_results}
\end{figure*}

\begin{figure}[t!]
    \centering
    \includegraphics[width=0.99\textwidth]{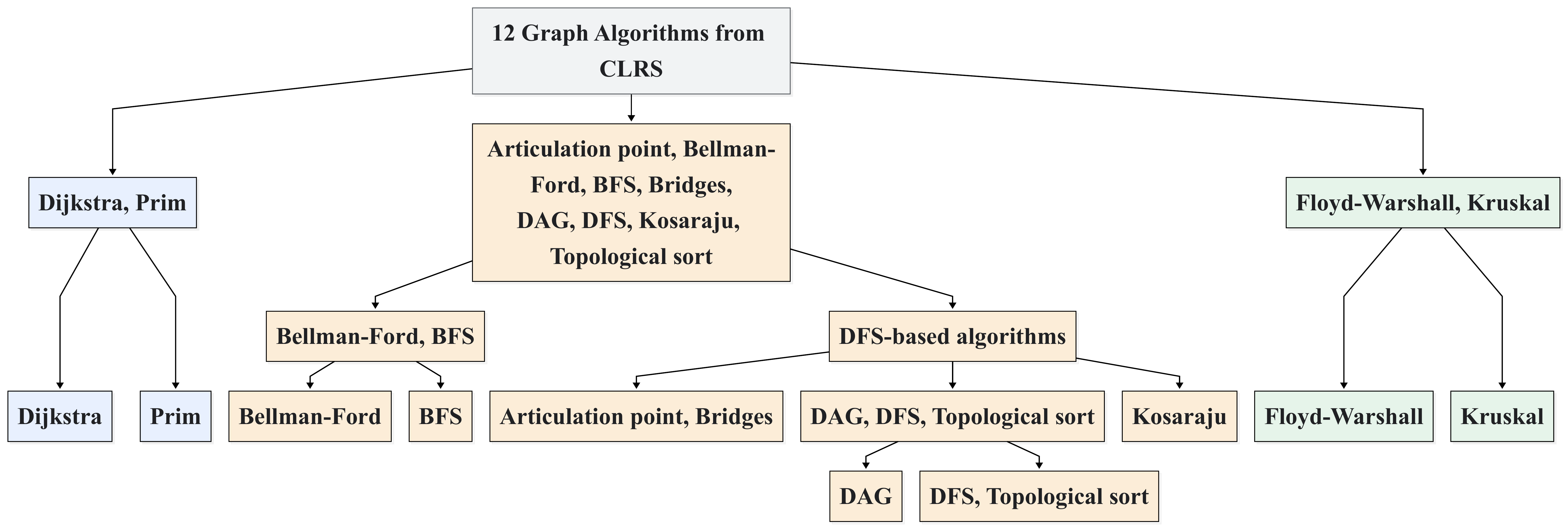}
    \caption{Illustration of the tree structure of the branching network learned on twelve algorithmic reasoning tasks from the CLRS benchmark. Our algorithm identifies clusters of algorithms that follow similar intermediate steps.
    }
    \label{fig_illustration_of_tree_clrs_tasks}
\end{figure}

Figure~\ref{fig_illustration_of_tree_clrs_tasks} illustrates the branching network structure discovered by \acronym{} when training Edge Transformers on the twelve graph-based algorithmic reasoning tasks.
The resulting structure aligns well with task similarities in their intermediate steps, revealing three major clusters. The largest includes BFS, Bellman-Ford, and several DFS-based algorithms. Notably, BFS and Bellman-Ford are grouped together, consistent with observations from~\cite{velivckovic2019neural}.
Five DFS-related tasks, including topological sort and DAG shortest paths, are clustered around DFS.
Prim’s and Dijkstra’s algorithms form a group, reflecting their shared greedy edge-selection strategy. Kruskal’s and Floyd-Warshall are grouped as well, both involving edge selection within components.
A similar structure is observed using MPNN, which we omit for brevity.

\begin{table*}[t!]
\centering
\caption{Test score (\%) evaluated on twelve graph algorithmic reasoning tasks. 
We report the average $F_1$-score over all tasks, GPU memory of the model, and runtime in terms of GPU hours, for each method. 
The test score is evaluated between model outputs and task labels at the final step.
For each experiment, we run the experiment with three random seeds and report the average.}\label{tab_full_clrs_benchmark_results}
\resizebox{\textwidth}{!}
{\begin{tabular}{@{}lcccccccccccc|ccc@{}}
\toprule
MPNN & STN & MTN & MMoE & TAG & LearningToBranch & GradTAG & \acronym{} \\ 
\midrule
BFS & 100.0 $\pm$ 0.0 & 98.0 $\pm$ 0.4 & 89.4 $\pm$ 0.2 & 99.4 $\pm$ 0.5 & 93.2 $\pm$ 0.4 & 99.4 $\pm$ 0.4 & 99.6 $\pm$ 0.3 \\
DFS & 38.7 $\pm$ 1.5 & 29.7 $\pm$ 1.6 & 30.2 $\pm$ 2.7 & 36.7 $\pm$ 1.4 & 34.4 $\pm$ 3.2 & 25.1 $\pm$ 4.5 & 38.7 $\pm$ 1.5 \\
Topological sort & 66.0 $\pm$ 2.6 & 71.3 $\pm$ 3.6 & 74.8 $\pm$ 2.4 & 72.8 $\pm$ 1.2 & 75.4 $\pm$ 2.1 & 61.4 $\pm$ 3.1 & 74.2 $\pm$ 1.7 \\
Articulation points & 99.8 $\pm$ 0.2 & 99.8 $\pm$ 0.1 & 97.1 $\pm$ 1.2 & 97.4 $\pm$ 1.3 & 95.5 $\pm$ 1.6 & 85.7 $\pm$ 0.8 & 93.0 $\pm$ 0.2 \\
Bridges & 82.6 $\pm$ 2.0 & 69.3 $\pm$ 0.4 & 91.1 $\pm$ 1.2 & 92.1 $\pm$ 1.9 & 92.2 $\pm$ 2.4 & 81.2 $\pm$ 2.5 & 96.2 $\pm$ 1.4\\
SCC Kosaraju & 93.1 $\pm$ 3.1 & 92.4 $\pm$ 4.7 & 93.5 $\pm$ 2.6 & 89.0 $\pm$ 1.3 & 93.4 $\pm$ 2.6 & 90.1 $\pm$ 3.5 & 92.0 $\pm$ 1.8\\
MST Kruskal & 69.9 $\pm$ 1.5 & 63.6 $\pm$ 2.0 & 66.2 $\pm$ 1.4 & 64.5 $\pm$ 0.8 & 68.0 $\pm$ 2.5 & 69.4 $\pm$ 1.1 & 66.4 $\pm$ 1.3 \\
MST Prim & 53.8 $\pm$ 1.3 & 54.3 $\pm$ 3.6 & 52.3 $\pm$ 1.4 & 50.1 $\pm$ 3.5 & 54.3 $\pm$ 2.7 & 95.3 $\pm$ 1.7 & 52.2 $\pm$ 2.4 \\
Dijkstra & 70.9 $\pm$ 2.3 & 69.3 $\pm$ 1.9 & 69.3  $\pm$ 2.4 & 67.3 $\pm$ 1.5 & 72.1 $\pm$ 1.4 & 78.3 $\pm$ 1.0 & 68.8 $\pm$ 1.7 \\
Bellman-Ford & 63.8 $\pm$ 4.1 & 59.2 $\pm$ 3.3 & 56.8 $\pm$ 1.6 & 61.9 $\pm$ 2.7 & 58.0 $\pm$ 3.4 & 85.9 $\pm$ 0.6 & 61.8 $\pm$ 2.4 \\
DAG shortest paths & 89.2 $\pm$ 1.4 & 90.0 $\pm$ 2.7 & 83.6 $\pm$ 1.6 & 89.8 $\pm$ 1.0 & 86.1 $\pm$ 2.4 & 88.6 $\pm$ 0.7 & 89.1 $\pm$ 1.5 \\
Floyd-Warshall & 69.4 $\pm$ 0.5 & 62.5 $\pm$ 0.7 & 56.1 $\pm$ 0.7 & 58.4 $\pm$ 1.7 & 58.1 $\pm$ 2.8 & 25.4 $\pm$ 1.4 & 63.1 $\pm$ 1.3 \\ \midrule
Avg. score & 75.2 & 71.6 & 71.7 & 73.1 & 73.7 & 73.5 & 74.6 \\ 
GPU hours & 3.2 & 3.3 & 33.4 & 	14.9 &  21.8 & 4.5 & 10.8	 \\
GPU memory & 49.6 & 11.7 & 38.2 & 38.8 & 26.8 & 38.8 & 18.9 \\
\midrule
Edge Transformers & STN & MTN & MMoE & TAG & LearningToBranch & GradTAG & \acronym{} \\ \midrule 
BFS & 99.7 $\pm$ 0.4 & 98.6 $\pm$ 0.1 & 100.0 $\pm$ 0.0 & 100.0 $\pm$ 0.0 & 100.0 $\pm$ 0.0 & 99.7 $\pm$ 0.0 & 99.8 $\pm$ 0.0 \\
DFS & 65.6 $\pm$ 1.6 & 51.4 $\pm$ 1.5 & 39.9 $\pm$ 2.6 & 34.1 $\pm$ 1.8 & 36.7 $\pm$ 1.8 & 37.6 $\pm$ 3.7 & 42.6 $\pm$ 2.4 \\
Topological sort & 98.7 $\pm$ 0.2 & 99.0 $\pm$ 0.9 & 97.7 $\pm$ 0.6 & 97.7 $\pm$ 0.7 & 98.7 $\pm$ 0.9 & 97.6 $\pm$ 2.3 & 96.2 $\pm$ 0.2 \\
Articulation points & 93.0 $\pm$ 3.6 & 96.4 $\pm$ 1.4 & 89.2 $\pm$ 2.3 & 89.2 $\pm$ 2.1 & 91.2 $\pm$ 1.4 &  99.0 $\pm$ 0.3  & 98.1 $\pm$ 1.0 \\
Bridges & 91.9 $\pm$ 0.3 & 92.2 $\pm$ 0.3 & 99.0 $\pm$ 0.0 & 99.2 $\pm$ 0.1 & 99.2 $\pm$ 0.2 &  93.5 $\pm$ 1.7 & 97.3 $\pm$ 0.2 \\
SCC Kosaraju & 65.8 $\pm$ 3.6 & 69.1 $\pm$ 2.9 & 64.2 $\pm$ 3.4 & 75.3 $\pm$ 3.4 & 76.7 $\pm$ 1.2 &  68.1 $\pm$ 2.8 & 63.9 $\pm$ 1.2 \\
MST Kruskal & 84.0 $\pm$ 1.4 & 82.2 $\pm$ 0.8 & 80.4 $\pm$ 1.8 & 80.4 $\pm$ 1.6 & 81.4 $\pm$ 1.8 & 79.5 $\pm$ 0.5 & 81.9 $\pm$ 0.3 \\
MST Prim & 93.0 $\pm$ 1.6 & 71.9 $\pm$ 1.4 & 85.5 $\pm$ 3.0 & 85.5 $\pm$ 1.8 & 87.1 $\pm$ 3.0 & 84.6 $\pm$ 1.0 & 92.8 $\pm$ 1.4 \\
Dijkstra & 91.9 $\pm$ 1.6 & 93.1 $\pm$ 3.8 & 94.0 $\pm$ 1.6 & 92.5 $\pm$ 1.1 & 93.7 $\pm$ 2.0 & 96.9 $\pm$ 0.4 & 98.7  $\pm$ 0.9\\
Bellman-Ford & 89.9 $\pm$ 1.4 & 90.3 $\pm$ 2.3 & 92.7 $\pm$ 1.7 & 90.2 $\pm$ 1.4 & 92.4 $\pm$ 1.9 & 90.4 $\pm$ 0.6 & 96.7 $\pm$ 1.4 \\
DAG shortest paths & 97.6 $\pm$ 0.8 & 96.1 $\pm$ 0.2 & 98.3 $\pm$ 0.7 & 98.3 $\pm$ 0.3 & 99.3 $\pm$ 0.4 & 94.7 $\pm$ 0.5 & 99.0 $\pm$ 0.6 \\
Floyd-Warshall & 61.5 $\pm$ 3.4 & 76.9 $\pm$ 2.3 & 77.7 $\pm$ 1.4 & 77.7 $\pm$ 1.5 & 78.7 $\pm$ 2.7 & 82.2 $\pm$ 0.7 & 82.8 $\pm$ 1.8 \\ \midrule
Avg. score & 86.0 & 83.6 & 84.9 & 85.0 & 86.4 & 85.4 & 87.3 \\
GPU hours & 30.4 & 19.6 &  107.5 & 220.5 & 148.5 & 60.8 & 77.2 \\
GPU memory & 126.1 & 30.6 & 79.9 & 104.7 & 66.2 & 104.7 & 49.2\\
\bottomrule
\end{tabular}}
\end{table*}

\begin{table}[t!]
\centering
\caption{Test accuracy (\%) evaluated on text-based graph reasoning tasks from the CLRS-Text, GraphQA, and GraphWiz benchmarks. We compare our approach with fine-tuning a single adapter (MTN) and the most competitive MTL baseline, GradTAG \cite{li2024scalable_kdd}. 
We evaluate the average test accuracy between outputs and task labels at every intermediate step. We then compute the average accuracy over all tasks. We run each experiment with three random seeds and report the average.
}\label{tab_full_text_graph_reasoning}
 \resizebox{0.6\textwidth}{!}
{\begin{tabular}{@{}lccclccccccccccc@{}}
\toprule
CLRS-Text & MTN & GradTAG & \acronym{} \\ 
\midrule
BFS & 79.3 $\pm$ 0.0 & 83.5 $\pm$ 0.4 & 
89.6 $\pm$ 0.2 \\
DFS & 51.1 $\pm$ 0.7 & 55.7 $\pm$ 0.8 & 62.5 $\pm$ 1.9  \\
Topological sort & 19.0 $\pm$ 1.0 & 18.8  $\pm$ 0.1 & 18.2 $\pm$ 0.4  \\
Articulation points & 96.9 $\pm$ 0.1 & 97.4 $\pm$ 0.3 & 98.4 $\pm$ 0.2 \\
Bridges & 73.4 $\pm$ 0.0 & 73.3 $\pm$ 0.8 & 73.2 $\pm$ 0.7  \\
SCC Kosaraju & 93.6 $\pm$ 0.2 & 94.1 $\pm$ 0.5 & 95.2 $\pm$ 0.1 \\
MST Kruskal & 98.5 $\pm$ 0.2 & 98.6 $\pm$ 0.2 & 98.9 $\pm$ 0.9 \\
MST Prim & 62.1 $\pm$  0.1 & 62.9 $\pm$ 0.3 & 64.2 $\pm$ 0.2  \\
Dijkstra & 63.6 $\pm$ 0.8 & 64.1 $\pm$ 0.7 & 65.0 $\pm$ 0.3  \\
Bellman-Ford & 74.8 $\pm$ 0.3 & 79.1 $\pm$ 0.8 & 85.0 $\pm$ 0.1 \\
DAG & 91.1 $\pm$ 0.9 & 89.1 $\pm$ 0.9 & 85.7 $\pm$ 0.2 \\
Floyd-Warshall & 29.3  $\pm$  0.5 & 34.5 $\pm$  0.9 & 43.0 $\pm$  0.9  \\ \hline
Avg. accuracy & 69.4 & 70.9 & 73.3 \\ \midrule
GraphQA & MTN & GradTAG & \acronym{} \\ \midrule
Edge existence & 96.7 $\pm$ 0.3 & 100.0 $\pm$ 0.0 & 99.6 $\pm$ 0.4 \\
Node degree & 98.1 $\pm$ 0.1 &  99.2 $\pm$ 0.3 & 99.7 $\pm$ 0.3 \\
Node count & 100.0 $\pm$  0.0 &  100.0 $\pm$ 0.0 & 100.0 $\pm$ 0.0 \\
Edge count & 67.2 $\pm$ 2.2 & 68.0 $\pm$ 2.4 & 73.1 $\pm$ 1.7  \\
Connected nodes & 99.5 $\pm$ 0.2 &  99.9 $\pm$ 0.1 & 100.0 $\pm$ 0.0 \\
Cycle check & 99.4 $\pm$ 0.2 &  99.4 $\pm$ 0.1 & 99.8 $\pm$ 0.2 \\
Disconnected nodes & 81.2 $\pm$ 1.7 & 81.2 $\pm$ 0.6 & 94.6 $\pm$ 0.8 \\
Reachability & 98.0 $\pm$ 0.2 & 97.7 $\pm$ 0.2 & 98.2 $\pm$ 0.9 \\
Shortest path & 86.6 $\pm$ 0.1 & 91.0 $\pm$ 1.0 &  90.4 $\pm$ 0.4 \\
Maximum flow & 48.1 $\pm$ 0.3 & 46.9 $\pm$ 0.8 & 47.7 $\pm$ 0.9  \\
Triangle counting & 60.8 $\pm$ 0.2 & 62.1 $\pm$ 0.4 & 61.8 $\pm$ 0.9  \\
Node classification & 94.9 $\pm$ 3.5 & 94.9 $\pm$ 0.5 & 99.1 $\pm$ 0.3 \\ \hline
Avg. accuracy  & 85.9 & 86.7 & 88.7 \\ \midrule 
GraphWiz & MTN & GradTAG & \acronym{} \\ \midrule
Cycle detection  & 45.6 $\pm$ 0.6 & 43.6 $\pm$ 0.5 & 43.8 $\pm$ 0.7 \\
Connectivity & 75.6 $\pm$ 0.7 &  76.2 $\pm$ 0.5 & 78.3 $\pm$ 0.8 \\
Bipartite graph & 62.7 $\pm$ 0.6 &  64.7 $\pm$ 0.7 & 65.6 $\pm$ 0.5 \\
Topological sort & 13.7 $\pm$ 0.3 & 20.1 $\pm$ 0.3 & 22.1 $\pm$ 0.3  \\
Shortest path & 24.7 $\pm$ 0.9 &  25.2 $\pm$ 0.2 & 21.6 $\pm$ 0.8 \\
Maximum triangle sum & 	29.3 $\pm$ 0.7 &  29.3 $\pm$ 0.5 & 29.6 $\pm$ 0.6 \\
Maximum flow &  10.2 $\pm$ 0.7 & 13.2 $\pm$ 0.2 & 15.7 $\pm$ 0.7 \\
Hamilton path & 48.5 $\pm$ 0.4 & 46.4 $\pm$ 0.5 & 47.8 $\pm$ 0.3 \\
Subgraph matching & 84.0 $\pm$ 0.9 & 83.1 $\pm$ 0.2 &  85.2 $\pm$ 0.8 \\ \hline
Avg. accuracy  & 43.8 & 44.7 & 45.5 \\ 
\bottomrule
\end{tabular}}
\end{table}

\subsubsection{Results on LLMs}
We now present results on text-based graph reasoning tasks.
\acronym{} is compared against MTN, which fine-tunes a single LoRA adapter across all tasks, and GradTAG~\cite{li2024scalable_kdd}, the strongest multitask baseline. We exclude LearningToBranch~\cite{guo2020learning}, as it is not applicable to text-based datasets.

On CLRS-Text, \acronym{} improves average test accuracy by \textbf{5.5}\% relative to MTN and by \textbf{3.2}\% over GradTAG. This highlights the advantage of the branching structure in capturing varying levels of task similarity.
To demonstrate broader applicability, we also evaluate on the GraphQA and GraphWiz datasets and observe similar gains. On GraphQA, \acronym{} surpasses MTN by \textbf{3.3}\% and GradTAG by \textbf{2.2}\% on average. On GraphWiz, it achieves gains of \textbf{3.8}\% over MTN and \textbf{1.8}\% over GradTAG.
Complete results are provided in Table~\ref{tab_full_text_graph_reasoning}.

\subsubsection{Results on Community Detection} \label{sec_ablation}
We then apply \acronym{} to a large multitask learning instance comprising 500 tasks on the community detection dataset. 
We compare \acronym{} with MTN, LearningToBranch \cite{guo2020learning}, and GradTAG \cite{li2024scalable_kdd}. 
To illustrate the relative improvement, we also compare our algorithm with a well-known community detection method, BigClam \cite{yang2013overlapping}.  

The results are shown in Table \ref{tab_community_detection}. We observe that our algorithm outperforms the most competitive MTL baselines by \textbf{28}\%. Compared to LearningToBranch, our algorithm uses \textbf{4.5}$\times$ less GPU hours. Compared to GradTAG, our algorithm uses \textbf{28}\% less GPU memory. 

\begin{table}[t!]
\centering
\caption{We report the test macro-averaged $F_1$-score over all community labels, the number of GPU hours, and GPU memory usage (in GB) of our algorithm, evaluated on the Orkut community detection dataset with 500 community node labeling tasks.}\label{tab_community_detection}
{\small\begin{tabular}{@{}lcccccccc@{}}
\toprule
    & Macro $F_1$-score & GPU hours & Memory \\ \midrule
    BigCLAM \cite{yang2013overlapping} & 22.69 $\pm$ 0.25 & \textbf{11.2} & 39.1 \\
    MTN \cite{wu2020understanding} & 27.24 $\pm$ 1.66 & 37.8 & \textbf{5.7}\\
    LearningToBranch \cite{guo2020learning} & 34.53 $\pm$ 3.50 & 366.6 & 91.2 \\ 
    GradTAG \cite{li2024scalable_kdd} & 38.77 $\pm$ 4.63 & 68.0 & 125.4 \\
    \midrule
    \acronym{} (Ours) & \textbf{49.69 $\pm$ 2.97} & 81.4 & 89.5\\
\bottomrule
\end{tabular}}
\end{table}

\subsubsection{Ablation Analysis}
Next, we discuss the key parameters in building the branching network, including the number of layers $L$ and the number of clusters $k$ in the final layer.
For GNNs, we tune $L$ on a single multitask network from $3$ to $6$ and select the one with the best average validation performance.
We set $L = 5$ for the CLRS benchmark and $L = 3$ for the community detection dataset.
For language models, $L$ is set as the depth of the pretrained language model, which is $28$ for Qwen-3-1.7B and $16$ for Llama-3-1B. We perform task partitioning every 4 layers. 

For the cluster size, we determine the number of clusters at each layer based on the average affinity score within clusters. We control the growth of the cluster size (within a factor of $5$) to prevent a dramatic expansion of the tree.
On the CLRS benchmark, this leads to $k = 10$ clusters in the final layer.
On the community detection dataset, there are $k=22$ clusters in the final layer.
We note that once task affinity scores are computed, clustering can be repeated within seconds to select the best configuration.
We report all hyperparameters, along with base model runtime and memory usage, in Appendix \ref{sec_experiment_details}.

Furthermore, we validate the branching structure by randomly varying the positions of tasks across branches. Compared to the tree in Figure \ref{fig_illustration_of_tree_clrs_tasks}, we find that swapping positions between branches, such as between Kruskal’s and BFS or DFS and Bellman-Ford, leads to more than 2\% test $F_1$-score drops.

\subsection{Measuring Sample Complexities}\label{sec_sc}

Finally, we evaluate the empirical sample complexity for learning three graph algorithms: BFS, Dijkstra's algorithm, and Prim's algorithm.
In particular, we set the target error rate to 1\% and increase the number of samples until the trained GNN's validation error falls below 1\%. We evaluate the error between model predictions and task labels at the last step of the algorithm.  
The results are shown in Figure~\ref{fig_graph_sc}. 

We find that the number of samples needed for learning the three algorithms differs, with Prim being the most difficult to learn.
We also compare training with and without the loss evaluated on the node labels of intermediate steps. We found that training with intermediate steps reduces the number of samples by approximately a factor of two, on both Dijkstra's and Prim's algorithms. 
Furthermore, we found that standard multitask training (MT) on Dijkstra and Prim tasks with a single GNN network increases the sample complexity of single-task training. It remains an interesting direction for future work to explain this observation. Specifically, what determines the learnability of an algorithm?
One hypothesis is that the higher sample complexity of learning the Prim algorithm comes from the (non-local) nature of its executions.%

\section{Related Work}\label{sec_related}

Recent work has examined whether neural networks can be trained to imitate classical algorithms step-by-step.
For graph algorithms, \citet{velivckovic2019neural} showed that message-passing neural networks with maximization aggregation closely align with the internal computations of algorithms such as reachability and shortest paths.
Their results further indicate that parameter sharing across layers yields positive transfer between related algorithms.
To systematically investigate this challenging problem, the CLRS benchmark was introduced \cite{velivckovic2022clrs}, providing a unified evaluation suite covering thirty algorithms from the standard textbook by \citet{cormen2022introduction}.

\begin{figure}[t!]%
    \centering
    \begin{minipage}[t]{0.33\textwidth}
        \centering
        \includegraphics[width = 0.875\textwidth]{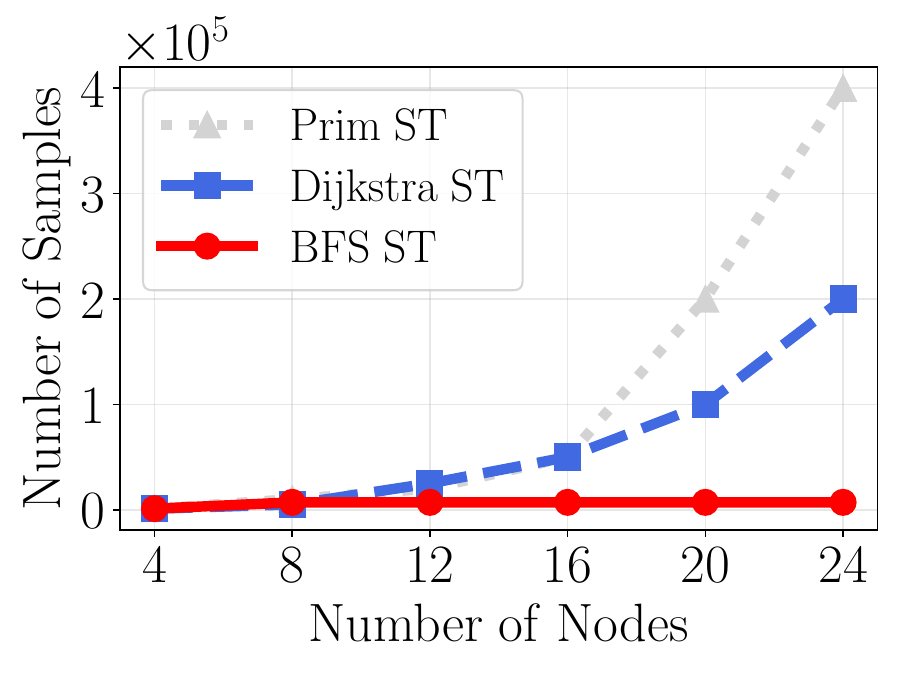}%
    \end{minipage}\hfill
    \begin{minipage}[t]{0.33\textwidth}
        \centering
        \includegraphics[width = 0.875\textwidth]{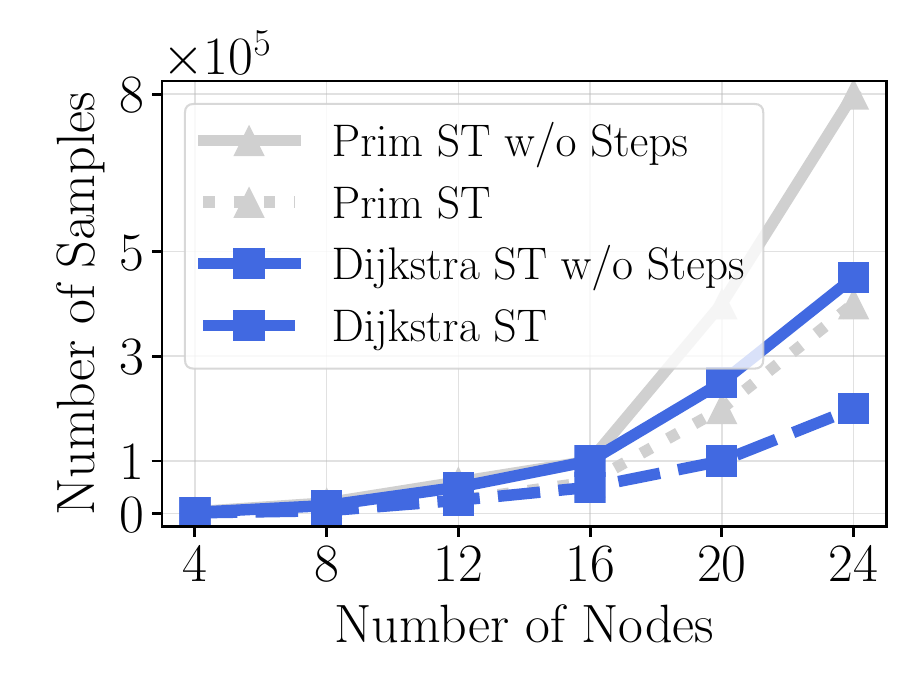}
    \end{minipage} \hfill
    \begin{minipage}[t]{0.33\textwidth}
        \centering
        \includegraphics[width = 0.875\textwidth]{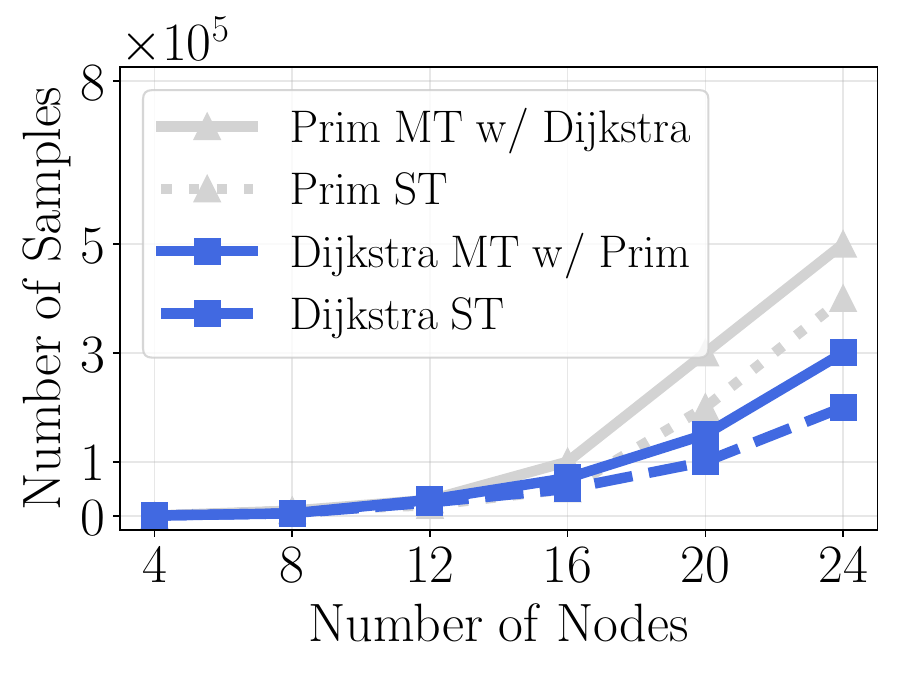}
    \end{minipage}        
    \caption{We illustrate the sample complexity scaling for learning three graph algorithms on graphs with different numbers of nodes. We find that the sample complexity for single-task training (ST) on the three algorithms differs when they are trained separately, with Prim requiring the most training samples. Second, training without the intermediate steps (w/o intermediate Steps) yields higher sample complexity compared to training with intermediate steps. 
    Further, with multitask training (MT), the number of samples increases compared with single-task training. In this experiment, we conduct multitask training on the combined dataset of Dijkstra and Prim tasks.
    }\label{fig_graph_sc}
\end{figure}

Several recent works have proposed architectures tailored for algorithmic reasoning.
\citet{ibarz2022generalist} introduce a triplet-GMPNN model that assigns values to edges and updates each edge based on the values of its adjacent edges.
When applied to tasks such as reachability and shortest paths, the network follows similar execution steps \cite{ibarz2022generalist}. %
More recently, augmenting transformers with tri-attention over node pairs, relative to other-conditioning, has been shown to substantially improve prediction accuracy on algorithmic graph tasks \cite{muller2024towards}.
Branching architectures have also been studied in computer vision \cite{lu2017fully,guo2020learning}. 
Branched multitask networks \cite{vandenhende2019branched} seek a tree-structured architecture that minimizes the sum of task-affinity scores across layers.

Another relevant direction is the design of mixture-of-experts models for multitask learning.
Such models comprise multiple expert subnetworks together with a gating mechanism that produces weighted combinations of expert outputs.
Examples include task-specific gating functions \cite{ma2018modeling} and differentiable sparse gating \cite{hazimeh2021dselect}.
A complementary line of work studies multitask learning from a regularization perspective \cite{wu2020understanding,yang2025precise}.
Recent analyses investigate how implicit and explicit regularization influence task interference and representation sharing \cite{liidentification,li2025efficient}, and apply gradient-based task features to supervised fine-tuning \cite{li2025efficient} and in-context learning \cite{zhang2025linear}, highlighting the role of gradient geometry in understanding task-relatedness.
Our work establishes a first connection between this body of literature and the problem of multitask algorithmic reasoning.

We now discuss recent studies evaluating the reasoning capabilities of LLMs on graph problems. GraphQA \cite{fatemitalk} systematically evaluates prompting strategies for encoding graphs (e.g., adjacency or incidence matrices) across twelve algorithms and retrieval tasks, finding that even large models such as PaLM-62B can underperform simple majority-vote baselines. 
NLGraph \cite{wang2023can} introduces eight graph reasoning tasks and shows that GPT-3.5 succeeds on problems like graph connectivity and shortest paths but struggles with more complex computations, such as Hamiltonian paths. GraphWiz \cite{chen2024graphwiz} applies instruction tuning with explicit reasoning traces, while GraphInstruct \cite{luo2024graphinstruct} develops instruction-tuned datasets over twenty-one graph reasoning tasks with detailed intermediate steps.  
It is worth clarifying that existing work focuses on \emph{predicting the final output of the algorithmic task}.
By contrast, we focus on \emph{models capable of generating the intermediate steps}.
Our fine-tuning experiments with Llama on GraphQA versus CLRS-Text further indicate that producing intermediate states is substantially more challenging---a gap that warrants further exploration.
CLRS-Text \cite{markeeva2024clrs} reformulates CLRS tasks as natural-language descriptions, showing that transformers can be trained to mimic algorithm execution, though they typically require significantly more training data than GNNs.
\citet{bounsi2024transformers} introduce a multimodal architecture that jointly pretrains transformers and GNNs using cross-attention to integrate their representations.
GraphArena \cite{tang2024grapharena} proposes broader benchmarks evaluating feasibility, hallucination, and other dimensions of graph-reasoning behavior.

Various techniques have been proposed to enhance LLM performance on graph-related problems. 
One class of approaches develops multimodal architectures that integrate GNNs with LLMs.
GreaseLM \cite{zhang2022greaselm} augments LLMs with knowledge-graph signals through a fusion layer that combines textual and graph-based representations.
GNP \cite{tian2024graph} incorporates a GNN module to encode knowledge graphs as embeddings within LLM prompts. %
Other lines of work explore code-generation pipelines (e.g., fine-tuning code LMs in graph coders), specialized tool-calling instructions (GraphTool-Instruction), and LLM-based multi-agent frameworks. 
Our work is complementary to these works and focuses on multitask algorithmic reasoning, with an emphasis on predicting not only final answers but also intermediate algorithmic states.
An interesting direction is to provide better explanations of GNNs for algorithmic reasoning, for instance by measuring the topology of complex predictions \cite{liu2023topological}, as well as to study broader algorithmic tasks including randomized computation \cite{tong2006fast,andersen2007local,kloster2014heat,zhang2016approximate}.

Besides algorithmic reasoning, latent multi-hop reasoning has also been studied to evaluate large language models in factual information retrieval \cite{yang2024large}. This line of work examines how models internally retrieve and utilize intermediate factual knowledge stored in their parameters to answer a question when such information is not explicitly provided in the prompt. Prior work \cite{biran2024hopping} identifies a mechanism for latent two-hop reasoning in which the first hop is resolved in earlier layers through identifying the intermediate answer, which then propagates to later layers to resolve the second hop. An interpretability method has been proposed to analyze failures in latent multi-hop reasoning by tracing how logits propagate across layers and positions \cite{yu2025back}. This analysis shows that errors can arise from conflicts among entity logits extracted in higher layers. In contrast, our work focuses on algorithmic reasoning, where models explicitly generate intermediate steps. It is intriguing to study whether branching neural networks can capture reasoning over multiple factual knowledge.

\section{Conclusion}

This work introduces a branching network capable of simultaneously solving multiple algorithmic reasoning tasks. We design an algorithm that automatically learns an optimal branching structure to capture the varying similarities between tasks. This algorithm efficiently finds the structure, enabled by a fast sub-procedure that determines the task partition at every layer. 
Across various algorithmic reasoning tasks and model architectures, our method outperforms existing multitask and branching networks, while achieving an optimal trade-off between runtime and model memory usage. 
Overall, these findings highlight the effectiveness of branching architectures for multitask algorithmic reasoning.

One open question is to flesh out the discrepancy in terms of the prediction results between breadth-first search and depth-first search (cf. Table \ref{tab_full_clrs_benchmark_results}).
Note that similar results are also observed in prior work (See Table 2 in the work of \cite{velivckovic2022clrs}).
One explanation for this discrepancy is that breadth-first search is more stable due to its local aggregation operation. In a follow-up work \cite{li2026quantifying}, we propose a new stability measure to quantify the error propagation in multistep algorithmic reasoning.
Our findings show that, first, the average product of Jacobian norms is substantially larger for DFS; second, DFS examples tend to have longer intermediate steps.
Another open question is to explore algorithmic reasoning in the context of other types of graph algorithms such as PageRank and random walks \cite{gleich2015pagerank}.

\section*{Acknowledgment}

Thanks to Wolfgang Gatterbauer, Jonathan Ullman,  David Gleich, Stratis Ioannidis, and Virgil Pavlu for several discussions about this work.
Thanks to Yu Cheng for pointing out an issue in Figure \ref{fig_task_examples}.
Thanks to the anonymous referees for their feedback.
The work of D. Li, Z. Zhang, M. Duan, and H. R. Zhang is partially funded by NSF award IIS-2412008.
D. Li is also partially funded by a PhD fellowship from JPMorgan Chase \& Co.

%% file: appendix.tex
\onecolumn
\input{proof}

\section{Detailed Approach}\label{sec_details_of_approach}

We list the notations we use in the paper as follows.
\begin{table}[h]
\caption{A list of mathematical notations used in the paper.}
\centering
\begin{tabular}{l l}
\toprule
Notation           & Meaning                                           \\ \midrule
${A}^{(i)}$      & The $i$-th algorithm                                  \\
$X$                & Input, including the graph structure                                \\
$A_j^{(i)}$ & $j$-th intermediate step of the $i$-th algorithm          \\
$S^{(i)}(X)$             & Total number of steps of applying the $i$-th algorithm on input $X$                            \\
$S$        & Subset of tasks \\
$f_W$ & A neural network with parameters $W$ \\
$f_W^{(i)}(X; j)$       & The prediction of the step $j$ of the $i$-th algorithm     \\
$g_j^{(i)}$ & The gradient of $f_W^{(i)}(X; j)$ with respect to the network parameter $W$\\
$T^{(i)}_{j, l}$    & Layer-wise task affinity between task $i$ and $j$ at layer $l$ \\
\bottomrule
\end{tabular}
\end{table}

\subsection{Clustering Algorithms}\label{app_clustering}

We provide further details of our fast approximation algorithm, which uses a first-order expansion of the network output and logistic regression on gradient-based features, combined with a dimension reduction step.
In the following discussion, we focus on binary classification, such that $A_j^{(i)}(X) \in \set{+1, -1}$.
Recall the gradient-based approximation of $f_W^{(i)}(X, j)$, given the input $X$ and intermediate step $j$: 
\begin{align}
    f_W^{(i)}(X, j)  \approx& f_{W^{(0)}}^{(i)}(X, j) + \big[ \nabla_{W_{l : L}}  f_{W^{(0)}}^{(i)}(X, j) \big ]^{\top} \big(W_{l : L} - W^{(0)}_{l : L} \big) + \epsilon^{(i)}_{X,j}.
\end{align} 
Let us denote $\nabla_{W_{l : L}}  f_{W^{(0)}}^{(i)}(X, j)$ as $g_j^{(i)}$ and $- A_j^{(i)}(X) f_{W^{(0)}}^{(i)}(X, j)$ as $b_j^{(i)}$, for any intermediate step $j$.
Using logistic loss, the approximate loss term for each sample is
\begin{align}
    \tilde\ell(f_W^{(i)}(X))= \log \Big( 1 + \exp \big( -A_j^{(i)}(X) g_j^{(i)\top} \big(W_{l : L} - W^{(0)}_{l : L} \big) + b_j^{(i)} \big) \Big),
\end{align}
for $W \in\real^p$.
Denote the combined data set in the task subset $S$ as $\cD_S$, 
where $n_S$ is the combined number of data samples in the set $\cD_S$.
The main idea is to solve a logistic regression problem with $g_j^{(i)}$ being the feature vector and $A_j^{(i)}(X)$ being the response label.

Since $g_j^{(i)}$ can be high-dimensional (on the order of the neural network's parameter count), we apply a Johnson-Lindenstrauss (JL) random projection to reduce dimension without losing much precision. Specifically, let $P\in\mathbb{R}^{p\times d}$ be a Gaussian random matrix, where each entry is drawn i.i.d. from $\cN(0, d^{-1})$.
We project the gradient from dimension $p$ onto dimension $d$ as
$\tilde g_j^{(i)} = P^{\top} g_j^{(i)}$.
Then, we solve the following $d$-dimensional logistic regression, which is now in dimension $d$, which aims to minimize the approximated logistic loss on all training samples for tasks in $S$: 
\begin{align*}
    \hat W_d \leftarrow \mathop{\arg \min}_{W\in\real^d}  \frac 1 {n_{S}} \sum_{X \in \cD_S} \sum_{i \in S} \frac{1}{S^{(i)}(X)} \sum_{j=1}^{S^{(i)}(X)} \log \Big( 1 + \exp \Big( - A_j^{(i)}(X) \big[g_j^{(i)}\big]^\top PW  + b_j^{(i)} \Big) \Big).
\end{align*}

Finally, we map the learned $\hat{W_d}\in\mathbb{R}^d$ back to the full $p$-dimensional space: $\hat W_{S}=P \hat W_d + W^{(0)}$. The $\hat W_{S}$ then is the estimation of the fine-tuned weights on a subset $S$.

\paragraph{Task Affinity based Clustering.} Let $T$ denote the task affinity matrix at a layer $l$. The goal is to cluster $n$ tasks into $k$ groups such that tasks with higher affinity are grouped together. 
Given the number of clusters $k$, let $v_1, \dots, v_k$ be binary indicator vectors indicating whether a task is in each cluster. The average density is computed as $\sum_{i=1}^k ({v_i^{\top} T v_i}/{v_i^{\top} v_i})$.
Specifically, define a new variable as $X = \sum_{i=1}^k ({v_i v_i^{\top}}/{v_i^{\top}v_i})$. As $v_iv_i^{\top}$ is a rank-one semi-definite matrix, the rank and trace of $X$ are equal to $k$.
This can be formulated as a rank-constrained maximization problem:
\begin{align*}
    \max_{X\in\real^{n\times n}}~~& \inner{T}{X} \\
        & X e = e, \tr[X] = k, \textup{rank}(X) = k \\
        & X \geq 0,  X \succeq 0.
\end{align*}

The integral objective is NP-hard in general. Therefore, we relax it using semi-definite programming (SDP), followed by a rounding step to obtain discrete clusters.

Since splitting tasks into $k$ clusters at layer $l$ increases the network size by $(L-l)k$, we apply the SDP relaxation with a regularization term on the number of clusters $k$ to control the network size.
The regularized SDP objective becomes:
\begin{align}\label{eq_clustering}
    \max_{X\in\real^{n\times n}}~~& \inner{T}{X} - \lambda (L-l) \tr[X] \\
        & X e = e, X \geq 0,  X \succeq 0 \notag
\end{align}
where $\lambda$ is a hyperparameter that controls how much we penalize the increase in network size.
Once we obtain the SDP solution $\hat X$, we round $\hat X$ into an integer solution using a threshold of $1/n$. 
Solving the SDP for an $n \times n$ matrix is computationally efficient, taking less than 5 seconds, making it negligible compared to the overall model training time.

\paragraph{Discussion.} 
Let the running time for training a single $L$-layer network on one task be $T$, with memory usage $B$. Our algorithm discovers the branching network in $nLT$ time and uses $kB/L$ memory.
By contrast, training a mixture of expert networks with $k$ networks takes $knT$ time and $kB$ memory. Task grouping based on fully computed task affinity takes $n^2T$ time and $kB$ memory. LearningToBranch \cite{guo2020learning} that searches for branching decisions takes $k^LnT$ time and uses $nB$ memory since it trains with $n$ modules per layer. 
One existing method \cite{lu2017fully} takes a layer-by-layer search approach from the last layer to the first. At each layer, their method determines the number of groups for the branching by optimizing a criterion over task-relatedness and model complexity.
In contrast, LearningToBranch \cite{guo2020learning} parameterizes the branching decisions at each layer and optimizes the variables in an end-to-end training pipeline. 
\citet{vandenhende2019branched} designs a method that exhaustively searches over all possible trees of $L$ layers to find the tree structure that has the minimum sum of task affinity scores over all layers.
These methods use a search space of $O(k^{n L})$. 
Other strategies have also been explored for tree construction. For example, one can still proceed from the last layer toward the first layer. Our algorithm can be applied generically to these different branching strategies, providing flexibility when building a branched multi-task network.

\subsection{Using Node Embeddings for Graph Algorithmic Reasoning}\label{sec_prompt_tuning}

Next, we use the embeddings learned from Algorithm \ref{alg_building_branching_network} to deal with text descriptions of a graph reasoning task.
To motivate our approach, we first evaluate whether existing open-source language models can solve the tasks accurately.
We use four CLRS-Text tasks, including BFS, Bellman-Ford, Dijkstra, and Prim's algorithm \cite{markeeva2024clrs}.
We use Llama-3-8B as the base language model.
The input involves a text description of a graph instance and the algorithmic task.
The output involves a text description of the intermediate steps and the final output.
The results are evaluated by comparing the outputs to the ground truth.

We find that there is a large gap between the LLM results and the GNN results.
We evaluate directly prompting the text description to the language model by showing $m$ demonstrations and also a query.
We also fine-tune Llama-3-8B with LoRA \cite{hu2021lora} based on the text description. 
For fine-tuning, we use the same number of training samples as in the MPNN. 
Table \ref{tab_graph_tasks} shows the results.
We find that directly prompting the Llama model performs poorly.
While fine-tuning helps, it still lags behind MPNNs by up to 23\%.

Can we combine the embeddings from the branching network along with a language model to enhance its graph reasoning performance?
A natural idea is to use the node embeddings of the input graph trained from \acronym{} and combine that with the embeddings of the text description.
To align these two embeddings, we add a linear layer after the node embeddings.
Then, these are used as the input to the language model for fine-tuning.
The overall procedure is summarized in Algorithm \ref{alg_prompt_tuning}.

Another benefit of using GNN embeddings is that they reduce the memory cost of LLMs by reducing the input lengths. 
Suppose a graph has $|V|$ nodes and $|E|$ edges, text-based descriptions of graphs use the flattened adjacency matrix (as in \cite{markeeva2024clrs}) and incidence matrix (as in \cite{fatemitalk}), 
which have a length of $|V|^2$ and $|V||E|$, respectively. 
As the model’s memory requirements typically scale with the square of these lengths, the overall memory usage becomes $O(|V|^4)$ and $O(|V|^2|E|^2)$.
In contrast, using GNN-produced node embeddings creates an input of length $|V|$, thereby reducing the memory requirement to $O(|V|^2)$. 

Other methods also exist for fusing the graph embeddings with the text embeddings.
For instance, one approach is to apply a cross-attention layer between the graph and text embeddings, and then feed the transformed embeddings to the LLM.
We find that training a cross-attention layer between the embeddings converges more slowly than directly training on the concatenated embeddings. Thus, we choose the simpler concatenation in our method.

\begin{algorithm}[t!]
    \caption{Fine-tuning LLMs with \acronym{} Embeddings}\label{alg_prompt_tuning}
    \raggedright
    \textbf{Input}: Training and validation datasets of $n$ algorithmic reasoning tasks in text formats, Pretrained language model $\text{LM}(\cdot)$ \\
    \textbf{Require:} A text instruction $P_{\text{instruction}}$,  Number $m$ of subsets and their size $\alpha$, Projection dimension $d$, Regularization parameter $\lambda$\\
    \textbf{Output}: A fine-tuned language model with an \acronym{} \\
    \begin{algorithmic}[1]
    \State Convert the text input descriptions into graphs for every task 
    \State Train an $L$-layer \acronym{} on the $n$ tasks with graph inputs, using Algorithm \ref{alg_building_branching_network} 
    \State For each training sample,
     concatenate the node embeddings of  \acronym{} with the text embeddings of $P_{\text{instruction}}$
    \State  Fine-tune the LM along with the \acronym{}, using the concatenated embeddings as input,
     and minimizing the training loss averaged over intermediate steps. 
    \end{algorithmic}
\end{algorithm}

\paragraph{Evaluation results.}

We observed that fine-tuning LMs on the text descriptions of tasks underperforms training GNNs. We now evaluate our algorithm that uses the embeddings of the GNN-based branching network to fine-tune LLMs. We use the text-based graph reasoning tasks from the CLRS-Text benchmark \cite{markeeva2024clrs}, training on 1,000 graphs and evaluating on validation and test sets, each containing 200 graphs. We use Llama-3-8B as the base model. For each task, we evaluate the accuracy between model outputs and the labels, averaged over intermediate steps 

We compare our algorithm with several parameter-efficient fine-tuning methods for fine-tuning Llama-3-8B. These include prefix tuning, adapter tuning, and LoRA fine-tuning. In our algorithm, we use LoRA to fine-tune the LM together with \acronym{}. For all baselines, we train the same number of parameters as in our algorithm.
Table \ref{tab_graph_tasks} shows the results. Our algorithm outperforms the other fine-tuning baselines by \textbf{2.1}\% on average. 

\begin{table}[t!]
\centering
\caption{We compare GNN training, prompt demonstration with $m$ examples, parameter-efficient fine-tuning, and our approach that leverages GNN embeddings in fine-tuning. For prompting and fine-tuning, we use Llama-3-8B. We evaluate the average test accuracy between outputs and labels at every intermediate step. We run each experiment with three random seeds and report the average.}\label{tab_graph_tasks}
 \begin{small}
{\begin{tabular}{@{}lcccccccccccc@{}}
\toprule
 & BFS & Bellman-Ford & Dijkstra & MST \\ \midrule
MPNN \cite{velivckovic2022clrs} & 99.80 & 99.40 & 99.20 & 99.60\\ \midrule  
Prompting, $m=0$ & 22.47 & 2.78 & 4.60 & 3.95\\
Prompting, $m=5$ & 34.28 & 31.54 & 25.17 & 25.31\\
Prompting, $m=10$ & 34.65 & 32.44 & 25.32 & 25.98\\
Prompting, $m=20$ & 35.88 & 33.56 & 25.80 & 26.59\\ \midrule
PrefixTuning & 92.1 $\pm$ 0.4 & 80.2 $\pm$ 0.6 & 65.3 $\pm$ 0.3 & 74.6 $\pm$ 0.8 \\
Adapter & 96.5 $\pm$ 0.6 & 83.1 $\pm$ 0.6 & 69.4 $\pm$ 0.4 & 74.3 $\pm$ 0.3 \\
LoRA & 98.2 $\pm$ 0.3 & 89.2 $\pm$ 0.2 & 75.6 $\pm$ 0.4 & 78.5 $\pm$ 0.5\\\midrule
Algorithm \ref{alg_prompt_tuning} (Ours) & \textbf{99.3 $\pm$ 0.2} & \textbf{92.6 $\pm$ 0.6} & \textbf{78.7 $\pm$ 0.6} & \textbf{79.4 $\pm$ 0.3}\\
\bottomrule
\end{tabular}}
 \end{small}
\end{table}

\section{Omitted Experiments} \label{sec_experiment_details}

\paragraph{CLRS Tasks.}
We summarize the input and output definitions of the algorithmic reasoning tasks in the CLRS benchmark \cite{velivckovic2022clrs} in Table~\ref{tab_clrs_tasks}. Each task represents an algorithmic problem where the input includes the graph structure itself.

\paragraph{GraphQA Tasks.} 
We use twelve text-based graph reasoning tasks from the GraphQA benchmark \cite{fatemitalk}. We summarize the input and output definitions of the tasks in Table \ref{tab_graphqa_tasks}. While not all tasks involve an algorithmic problem, we evaluate on this benchmark to demonstrate the broad applicability of our algorithm. 

\paragraph{GraphWiz Tasks.} We use nine graph-related tasks from GraphWiz \cite{chen2024graphwiz}. Compared to other datasets, this dataset provides textual descriptions of graph problem inputs and the reasoning processes that describe the rationale for solving the tasks generated by chain-of-thought prompting GPT-4. We summarize the input and output definitions of the tasks in Table \ref{tab_graphwiz_tasks}. 

\paragraph{Implementation.}  On CLRS, we use the encode-process-decode architecture introduced in \citet{velivckovic2022clrs}.
The encoder embeds inputs into feature representations of the algorithm. 
The processor performs message-passing computations of the input embeddings over the graph using GNNs.
The decoder is a linear layer that transforms the final embeddings into a label space. The output of the decoder is viewed as the predicted steps for the next step and as the output at the final step. 
We train the encode-process-decode model with the sum of the losses computed on each step. 
At each algorithm step $t$, the encoder generates embeddings for the nodes, edges, and graph-level features. The processor updates node embeddings, incorporating both the initial input features and the embeddings from previous steps. Finally, the decoder predicts intermediate steps or the final output step. The entire model is trained to optimize predictions across all steps.

\begin{table*}[t!]
\centering
\caption{Definition of input and output of 12 graph reasoning tasks from GraphQA \cite{fatemitalk}. The input graphs are sampled from an Erd\"os-R\'enyi random graph distribution with edge sampling probability $p = 0.5$, with 20 vertices in total on each graph. The edge weights are uniformly sampled from the integers between $1$ and $10$. The source node is sampled randomly from the vertex set.
}\label{tab_graphqa_tasks}
\resizebox{\textwidth}{!}
{
{
\begin{tabular}{@{} p{4cm} p{5cm} p{5.5cm} p{6.5cm} @{}}
\toprule
Task & Input Graph & Additional Input & Output \\ 
\midrule
Edge existence & Undirected & Two target nodes & Binary label for edge existence \\ 
Node degree & Undirected & Target node & Scalar for node degree \\ 
Node count & Undirected & None & Scalar for number of nodes \\ 
Edge count & Undirected & None & Scalar for number of edges \\ 
Connected nodes & Undirected & Target node & List of connected nodes \\ 
Cycle check & Undirected & None & Binary label for cycle presence \\ 
Disconnected nodes & Undirected & Target node & List of disconnected nodes \\ 
Reachability & Undirected & Source and target nodes & Binary label for path existence \\ 
Shortest path & Undirected, weighted & Source and target nodes &  Scalar for shortest path length \\ 
Maximum flow & Undirected, weighted & Source and target nodes &  Scalar for maximum flow capacity \\ 
Triangle counting & Undirected & None &  Scalar for number of triangles \\ 
Node classification & Undirected & Target node and other node labels & Node label of the target node \\
\bottomrule
\end{tabular}
}
}
\end{table*}

\begin{table*}[t!]
\centering
\caption{Definition of input and output of 9 graph reasoning tasks from GraphWiz \cite{chen2024graphwiz}. The input graphs are sampled from an Erd\"os-R\'enyi random graph distribution with edge sampling probability $p = 0.5$, with 100 vertices in total on each graph. The edge weights are uniformly sampled from the integers between $1$ and $10$. The source node is sampled randomly from the vertex set.}\label{tab_graphwiz_tasks}
\resizebox{\textwidth}{!}
{
{
\begin{tabular}{@{} p{4cm} p{5cm} p{5.5cm} p{6.5cm} @{}}
\toprule
Task & Input Graph & Additional Input & Output \\ 
\midrule
Cycle detection & Undirected & None & Binary label for cycle presence \\ 
Connectivity & Undirected & Two target nodes & Binary label for path existence \\ 
Bipartite graph & Undirected & None & Binary label for bipartite structure \\ 
Topological sort & Directed acyclic & None & Node labels for the ordering of nodes \\ 
Shortest path & Undirected, weighted & Source and target nodes & List of node labels for the shortest path \\ 
Maximum triangle sum & Undirected, weighted & None & Scalar for maximum triangle sum \\ 
Maximum flow & Directed, weighted & Source and sink nodes & Scalar for maximum flow capacity \\ 
Hamilton path & Undirected & None & Binary label for path existence \\ 
Subgraph matching & Two undirected graphs & None & Binary label for subgraph isomorphism \\ 
\bottomrule
\end{tabular}
}
}
\end{table*}

\subsection{Omitted Ablation Results}\label{sec_additional_approximation_results}

\paragraph{First-order approximation.} In Table \ref{table_compare_approximation_error_layerwise}, we measure the first-order approximation error by freezing different numbers of bottom layers of the models. 
Linear approximation consistently yields under 1\% error, for edge transformers \cite{muller2024towards} and MPNN \cite{velivckovic2022clrs}.
The approximation becomes better in higher layers. The RSS falls below 0.06\% after freezing the first three layers.  

\begin{table}[t!]
\centering
\caption{Measuring the error for approximating the algorithmic reasoning task loss, using the first-order Taylor's expansion around the initialization trained on all tasks. The results are averaged over $50$ random subsets of graph algorithmic reasoning tasks. We apply the approximation to parameters after layer $l$ and use MPNN \cite{velivckovic2022clrs} or edge transformer \cite{muller2024towards} as the base model.}\label{table_compare_approximation_error_layerwise}
 \resizebox{1.00\textwidth}{!}
{\begin{tabular}{@{}cccccccccccc@{}}
\toprule
& \multicolumn{3}{c}{RSS: Base model $=$ MPNN} & & \multicolumn{3}{c}{RSS: Base model $=$ Edge Transformer}  \\ \cmidrule(l){1-4}\cmidrule(l){5-8}
    Dist. & {Freeze layer $1$} &  {Freeze layer $1, 2$} & {Freeze layer $1, 2, 3$} & Dist. & {Freeze layer $1$} &  {Freeze layer $1, 2$} & {Freeze layer $1, 2, 3$} \\ \midrule
    2\% & 2.3$_{\pm 0.7} \times 10^{-3}$ & 2.0$_{\pm 0.6} \times 10^{-3}$ & 2.3$_{\pm 0.4} \times 10^{-5}$ & 2\% &  3.9$_{\pm 0.4} \times 10^{-3}$ & 3.0$_{\pm 0.2} \times 10^{-3}$ & 4.1$_{\pm 0.2} \times 10^{-5}$ \\
    4\% & 7.0$_{\pm 0.8} \times 10^{-3}$ & 6.0$_{\pm 1.7} \times 10^{-3}$ & 9.4$_{\pm 1.7} \times 10^{-5}$ & 4\% & 6.5$_{\pm 0.5} \times 10^{-3}$ & 6.4$_{\pm 1.0} \times 10^{-3}$ & 8.6$_{\pm 2.1} \times 10^{-5}$ \\
    6\% & 8.2$_{\pm 1.4} \times 10^{-3}$ & 8.0$_{\pm 1.2} \times 10^{-3}$ & 2.1$_{\pm 0.4} \times 10^{-4}$ & 6\% & 7.7$_{\pm 1.6} \times 10^{-3}$ & 7.4$_{\pm 0.9} \times 10^{-3}$ & 3.4$_{\pm 0.7} \times 10^{-4}$ \\
    8\% & 8.6$_{\pm 1.2} \times 10^{-3}$ & 9.1$_{\pm 1.6} \times 10^{-3}$ & 3.7$_{\pm 0.7} \times 10^{-4}$ & 8\% & 8.4$_{\pm 2.4} \times 10^{-3}$ & 8.5$_{\pm 2.0} \times 10^{-3}$ & 4.5$_{\pm 1.8} \times 10^{-4}$\\
    10\% & 9.2$_{\pm 2.4} \times 10^{-3}$ &  9.4$_{\pm 2.1} \times 10^{-3}$ &  5.8$_{\pm 1.1} \times 10^{-4}$ & 10\% & 9.6$_{\pm 1.1} \times 10^{-3}$ &  9.0$_{\pm 0.8} \times 10^{-3}$ &  6.7$_{\pm 2.2} \times 10^{-4}$ \\
\bottomrule
\end{tabular}}
\end{table}

\paragraph{Depth-first search.} To understand the limitations of graph neural networks on Depth-First Search (DFS), we conduct an empirical case study using star graphs. We hypothesize that standard GNNs would fail to learn the DFS task on star graphs, because the algorithm requires predicting distinct node labels, whereas standard message-passing aggregations cannot distinguish between leaf nodes. To test this, we train a GIN model with the max aggregation operation and evaluate the final-step node prediction accuracy on star graphs (a central node connected to leaf nodes). We vary the number of nodes between 10, 20, 50, and 100, and the number of layers between 1, 2, and 3. We use 2,000 graphs in the training dataset.

To test our hypothesis, we evaluate two types of input node-level features: one where all nodes are initialized with identical input features, and another where node indices are included in the input features. As shown in Table \ref{tab_dfs_ablation}, when all nodes are provided with identical features, the GNN model cannot fully predict the node labels. When the input features encode node indices, a 1-layer model achieves near 100\% test accuracy. This is because executing DFS on the star graphs requires identifying the center node (which the model infers from its high degree) and subsequently visiting the leaf nodes in increasing index order. Including the node indices in the input features enables the model to learn this ordering.

Furthermore, performance degrades as the number of layers increases. While the 1-layer model achieves near 100\% accuracy, the 2-layer and 3-layer models perform progressively worse. This is because two layers aggregate features across the entire star graph and leave the model unable to accurately differentiate the leaf node labels.

\begin{table}[htbp]
\centering
\caption{Test accuracy (\%) evaluated at the final step of the DFS task on star graphs. We train a GIN model using varying graph sizes and number of layers, comparing node input features with and without node indices. $l$ denotes the number of layers, and $n$ denotes the number of nodes. We report the average accuracy and the standard deviations across three random seeds.}
\label{tab_dfs_ablation}
{\small\begin{tabular}{llccccc}
\toprule
 &  & {$n=10$} & {$n=20$} & {$n=50$} & {$n=100$} \\
\midrule
\multirow{3}{*}{\shortstack[l]{Using input features\\including node indices}} 
 & $l=1$ & 100.0 $\pm$ 0.0 & 98.8 $\pm$ 1.1 & 98.1 $\pm$ 2.2 & 95.9 $\pm$ 2.1 \\
 & $l=2$ & 100.0 $\pm$ 0.0 & 66.1 $\pm$ 1.5 & 61.0 $\pm$ 2.0 & 57.0 $\pm$ 2.0 \\
 & $l=3$ & 100.0 $\pm$ 0.0 & 64.7 $\pm$ 1.0 & 68.7 $\pm$ 1.1 & 67.4 $\pm$ 0.9 \\
\midrule
\multirow{3}{*}{\shortstack[l]{Using input features\\without node indices}} 
 & $l=1$ & 64.3 $\pm$ 0.0 & 58.0 $\pm$ 0.0 & 53.5 $\pm$ 0.0 & 52.9 $\pm$ 0.0 \\
 & $l=2$ & 64.3 $\pm$ 0.0 & 58.0 $\pm$ 0.0 & 53.5 $\pm$ 0.0 & 52.9 $\pm$ 0.0 \\
 & $l=3$ & 64.3 $\pm$ 0.0 & 58.0 $\pm$ 0.0 & 53.5 $\pm$ 0.0 & 52.9 $\pm$ 0.0 \\
\bottomrule
\end{tabular}}%
\end{table}

\paragraph{Ablation analysis on CLRS-Text.}
Next, we conduct an empirical evaluation on graph algorithmic reasoning datasets, as studied in the CLRS benchmark \cite{velivckovic2022clrs,markeeva2024clrs}. We use twelve text-based graph algorithmic datasets, including the Bellman-Ford algorithm, BFS, and DFS. 
The input sample is encoded as a flattened adjacency matrix of the graph, and intermediate steps are encoded as lists of node labels. 
We fine-tune the pretrained Qwen-3-1.7B model on graphs of 10 nodes, using LoRA \cite{hu2021lora} as the base fine-tuning method.
We evaluate the error rates as one minus the average exact match across output text sequences.

First, we find that the sample complexity of generating intermediate steps can be higher than predicting only the final step. We compare fine-tuning with intermediate steps and fine-tuning with only the final step, in terms of the errors of the final step. 
We observe that under low-sample regimes (e.g., within 1,000 samples for the Bellman-Ford algorithm), fine-tuning with intermediate steps performs worse. With sufficient samples, the intermediate hints can still be beneficial. %
Second, we find that training on certain combinations of datasets can either positively or negatively affect the average sample complexity. 
Consider the example shown in Figure \ref{fig_task_examples}. The Bellman-Ford algorithm shares all steps with BFS on the graph, whereas DFS shares only the first step with BFS. One might expect that combining the datasets of Bellman-Ford and BFS would improve their average sample complexity. 
We fine-tune a model on each pair of datasets and compare it with training on each dataset alone. 
We find that for Bellman-Ford and BFS, joint training yields $4.8\%$ lower error rates on average across multiple sample sizes. By contrast, for Bellman-Ford and DFS, standalone training yields $12.1\%$ higher error rates on average.
In all $\binom{12}{2} = 66$ cases, we find 34 cases where joint training improves sample complexity.

%% file: proof.tex
\section{Proof of Proposition \ref{prop_error_bound}}

We provide a proof for the Proposition \ref{prop_error_bound} using the logistic loss in binary classification. We note that the extension to multiclass classification loss is straightforward using the same technique, which requires additional notations. 

\begin{proof}[Proof of Proposition \ref{prop_error_bound}]
Recall that the estimated weight $\hat{W}_S$ is obtained from the minimizer of the logistic regression using the projected gradients as the features. 
To make it clear, we annotate the vector with its dimension so that it is easy to distinguish. 
Let $\hat{W}_d$ denote the minimizer of the logistic regression in dimension $d$. We have $\hat{W}_S = P \hat {W}_d +W^{(0)}$ given a $p$ by $d$ random projection matrix $P$ and the weight initialization $W^{(0)}$. 

Specifically, $\hat W_d$ is the minimizer of the following problem:
\begin{align*}
    h_1(W) = \frac{1}{\abs{T}\abs{S}} \sum_{X \in T} \sum_{i \in S} \frac{1}{S^{(i)}(X)} \sum_{j=1}^{S^{(i)}(X)} \log \Big( 1 + \exp \Big( - A_j^{(i)}(X) \big[g_j^{(i)}\big]^\top PW  + b_j^{(i)} \Big) \Big),
\end{align*}
for $W \in \real^{d}$, where $g_j^{(i)} = \nabla f_{W^{(0)}}^{(i)}(X; j),  b_j^{(i)} = -A_j^{(i)}(X) f_{W^{(0)}}^{(i)}(X; j).$

To relate this loss with the training loss $\hat{L}_S(W)$, we define an intermediate solution $\bar{W}_p$ in dimension $p$ for the following problem: 
\begin{align*}
    h_2(W) = \frac{1}{\abs{T}\abs{S}} \sum_{X \in T} \sum_{i \in S} \frac{1}{S^{(i)}(X)} \sum_{j=1}^{S^{(i)}(X)} \log \Big( 1 + \exp \Big( - A_j^{(i)}(X) \big[g_j^{(i)}\big]^\top P P^{\top} (W-W^{(0)}) \big) + b_j^{(i)} \Big) \Big).
\end{align*}

We can know that $h_1(\hat W_d) \leq h_2(\bar{W}_p)$, since the second problem is a specific case of the first one. 
Denote $W^{\star}$ as the minimizer for the training loss for $W \in \real^{p}$: 
\begin{align*}
    \min \hat{L}_S(W) = \frac{1}{\abs{T}\abs{S}}\sum_{X \in T} \sum_{i\in S} \frac{1}{S^{(i)}(X)} \sum_{j=1}^{S^{(i)}(X)  } \log \Big( 1 + \exp \Big( - A_j^{(i)}(X) f_W^{(i)}(X; j) \Big)  \Big).
\end{align*}

We will complete the proof by showing that the following two bounds hold. First, we will show that the error between $h_1(\hat W_d)$ and $\hat{L}_S(W^{\star})$ is bounded:
\begin{align}\label{eq_bound_1}
    h_1(\hat W_d) \leq h_2(\bar W_p) \leq h_2(W^{\star}) \leq \hat{L}_S(W^\star) + \delta + 2GD\epsilon. 
\end{align}
Second, we will show that the error between $h_1(\hat W_d)$ and $\hat{L}_S(P \hat W_d + W^{(0)})$ is also bounded:
\begin{align}\label{eq_bound_2}
 \Big \lvert h_1(\hat W_d) - \hat{L}_S(P \hat W_d + W^{(0)}) \Big \rvert \leq \delta.
\end{align}

Next, we prove the first bound in Equation \eqref{eq_bound_1}. We have known that $h_1(\hat W_d) \leq h_2(\bar{W}_p)$ and $h_2(\bar{W}_p) \leq h_2(W^\star)$, as $\bar{W}_p$ is minimizer of $\min h_2(W)$. We will bound the error between $h_2(W^\star)$ and $\hat{L}_S(W^\star)$. 
Let's expand the training loss at $W^\star$. To simplify the notations, we write down one logistic loss:
\begin{align*}
    & \log \Big( 1 + \exp \Big( - A_j^{(i)}(X) f_{W^\star}^{(i)}(X; j) \Big)  \Big) \\
    = & \log \Big( 1 + \exp \Big( - A_j^{(i)}(X) \big[g_j^{(i)}\big]^\top (W^\star - W^{(0)})  + b_j^{(i)} + \bar \epsilon_j^{(i)} \Big) \Big) \\
    = & \log \Big( 1 + \exp \Big( - A_j^{(i)}(X) \big[g_j^{(i)}\big]^\top P P^{\top} (W^\star - W^{(0)})  + b_j^{(i)} + \bar \epsilon_j^{(i)} + \tilde{\epsilon}_j^{(i)} \Big) \Big),
\end{align*}
where $\bar \epsilon_j^{(i)} = -A_j^{(i)}(X) \epsilon_{X, j}^{(i)}$ involves the Taylor's expansion error and \[ \tilde \epsilon_j^{(i)} = \big[g_j^{(i)}\big]^\top(\id_d - PP^{\top}) (W^\star - W^{(0)}). \]
Then, we leverage the fact that the logistic loss is $1$-Lipschitz continuous. Based on the mean value theorem, we can have that: 
\begin{align*}
    \log(1+\exp(-x + \epsilon)) \leq \log(1+\exp(-x)) + \abs{\epsilon}. 
\end{align*}  
Therefore, we can bound the error between $h_2(W^\star)$ and $\hat{L}_S(W^\star)$ as follows: 
\begin{align*}
    h_2(W^\star) \leq \hat{L}_S(W^\star) + \frac{1}{\abs{T}\abs{S}} \sum_{X\in T}\sum_{i \in S} \frac{1}{S^{(i)}(X)} \sum_{j=1}^{S^{(i)}(X)} (\abs{\bar \epsilon_j^{(i)}} + \abs{\tilde \epsilon_j^{(i)}}).
\end{align*}
From the assumption that the averaged Taylor’s expansion error is at most $\delta$, we have:
\begin{align*}
        \frac{1}{\abs{T}\abs{S}} \sum_{X\in T}\sum_{i \in S} \frac{1}{S^{(i)}(X)} \sum_{j=1}^{S^{(i)}(X)} \abs{\bar \epsilon_j^{(i)}} \leq \delta.
\end{align*}
We then bound the second error term by the Johnson-Lindenstrauss Lemma \cite{johnson1984extensions}. Provided that $d = O\big(\frac{\log p}{\epsilon^2}\big)$, we have
    \begin{align*}
        \abs{\inner{g_i}{W^\star  - W^{(0)}} - \inner{P g_i}{P (W^\star  - W^{(0)})}} 
        \le \epsilon \bigabs{\inner{g_i}{W^\star  - W^{(0)}}}
        \le {2 G D }{\epsilon}.
    \end{align*} 
Therefore, we have proved the first bound in equation \eqref{eq_bound_1}: 
\begin{align*}
    h_2(W^\star) \leq \hat{L}_S(W^\star) + \frac{1}{\abs{T}\abs{S}} \sum_{X\in T}\sum_{i \in S} \frac{1}{S^{(i)}(X)} \sum_{j=1}^{S^{(i)}(X)} (\abs{\bar \epsilon_j^{(i)}} + \abs{\tilde \epsilon_j^{(i)}}) \leq \delta + 2GD\epsilon.
\end{align*}
Next, we prove the second bound in Equation \eqref{eq_bound_2} using the same idea. This is based on observing the training loss on $P\hat W_d + W^{(0)}$: 
\begin{align*}
    & \log \Big( 1 + \exp \Big( - A_j^{(i)}(X) f_{P\hat W_d + W^{(0)}}^{(i)}(X; j) \Big)  \Big) \\
    = & \log \Big( 1 + \exp \Big( - A_j^{(i)}(X) \big[g_j^{(i)}\big]^\top P\hat W_d  + b_j^{(i)} + \bar \epsilon_j^{(i)} \Big) \Big),
\end{align*}
where $\bar \epsilon_j^{(i)}$ involves the Taylor expansion error.

Using the $1$-Lipschitz continuous property of the logistic loss again, we can have that 
\begin{align*}
    \Big \lvert h_1(\hat W_d) - \hat{L}_S(P \hat W_d + W^{(0)}) \Big \rvert 
    \leq \frac{1}{\abs{T}\abs{S}} \sum_{X\in T}\sum_{i \in S} \frac{1}{S^{(i)}(X)} \sum_{j=1}^{S^{(i)}(X)} \abs{\bar \epsilon_j^{(i)}}
    \leq \delta.
\end{align*}
We have now completed the proof. 
\end{proof}